\documentclass[conference]{IEEEtran}
\usepackage{graphicx}
\usepackage{epsfig,subfigure}
\usepackage{amssymb,amsmath}
\usepackage{enumerate}
\usepackage{multirow}
\usepackage{array}
\usepackage{cite}
\usepackage{mathrsfs}
\usepackage{amsmath}
\newtheorem{theorem}{\textbf{Theorem}}
\newtheorem{lemma}{\textbf{Lemma}}

\newtheorem{assumption}{\textbf{Assumption}}
\newtheorem{corollary}{\textbf{Corollary}}

\usepackage[noend]{algpseudocode}
\usepackage{algorithmicx,algorithm}
\usepackage{soul}
\soulregister\cite7 
\soulregister\ref7 
\usepackage{color}

\begin{document}
\title{
To Talk or to Work: Flexible Communication Compression for Energy Efficient Federated Learning over Heterogeneous Mobile Edge Devices}

\author{
\IEEEauthorblockN{Liang Li\authorrefmark{1}\authorrefmark{2}, Dian Shi\authorrefmark{3}, Ronghui Hou\authorrefmark{1}\authorrefmark{2}, Hui Li\authorrefmark{1}, Miao Pan\authorrefmark{3}, and Zhu Han\authorrefmark{3}}
\IEEEauthorblockA{\authorrefmark{1}School of Cyber Engineering, Xidian University, Xi'an, China}
\IEEEauthorblockA{\authorrefmark{2}State Key Laboratory of Computer Architecture, Institute of Computing Technology, Chinese Academy of Sciences, China}
\IEEEauthorblockA{\authorrefmark{3}Department of Electrical and Computer Engineering, University of Houston, Houston, TX, USA}
\IEEEauthorblockA{Email: liliang\_1127@outlook.com, dshi3@uh.edu, rhhou@xidian.edu.cn, lihui@mail.xidian.edu.cn, \{mpan2, zhan2\}@uh.edu}
}

\maketitle

\begin{abstract}
Recent advances in machine learning, wireless communication, and mobile hardware technologies promisingly enable federated learning (FL) over massive mobile edge devices, which opens new horizons for numerous intelligent mobile applications. Despite the potential benefits, FL imposes huge communication and computation burdens on participating devices due to periodical global synchronization and continuous local training, raising great challenges to battery constrained mobile devices. In this work, we target at improving the energy efficiency of FL over mobile edge networks to accommodate heterogeneous participating devices without sacrificing the learning performance. To this end, we develop a convergence-guaranteed FL algorithm enabling flexible communication compression. Guided by the derived convergence bound, we design a compression control scheme to balance the energy consumption of local computing (i.e., ``working'') and wireless communication (i.e., ``talking'') from the long-term learning perspective. In particular, the compression parameters are elaborately chosen for FL participants adapting to their computing and communication environments. Extensive simulations are conducted using various datasets to validate our theoretical analysis, and the results also demonstrate the efficacy of the proposed scheme in energy saving.


\end{abstract}

\IEEEpeerreviewmaketitle

\section{Introduction}
The growing prevalence of mobile smart devices and the rapid advancement of social networking applications result in the phenomenal growth of the data generated at the edge network. To draw useful information from such geographically distributed data, federated learning (FL) has emerged as a promising paradigm that allows participating users to collaboratively learn a shared model, while keeping all the private training data on their edge devices. In particular, all the participants are allowed to run stochastic gradient descent (SGD) locally and send the intermediate gradients to the server periodically for global synchronization. The recent advances in mobile edge computing further facilitate the implementation of FL in mobile networks, since modern smart mobile devices are now armed with high-performance central processing units (CPUs) and graphics processing units (GPUs) to handle intensive computations of intelligent applications. With the technical advantages and implemental feasibilities, FL has seen recent successes in several applications, including next word prediction in Google’s Gboard \cite{hard2018federated}, vocal classifier for “Hey Siri” \cite{siri}, mobile augmented reality \cite{chen2020federated}, etc.

However, to practically deploy FL in wireless networks still faces several critical challenges. On the one hand, both transmitting the gradient updates and performing the local optimizations are resource-hungry, leading to considerable energy consumption at mobile edge devices during the training process. Despite improving computing capacity, mobile edge devices are generally subject to the limited battery lifetime, which hinders their applications in training complex models and supporting continuous learning. On the other hand, the mismatch between the heavy communication loads and the constrained wireless bandwidth hampers the efficient exchange of locally computed updates. The current trend of going deeper in the depth of neural networks has resulted in high-dimensional models with millions of parameters, which inevitably involves significant wireless traffic in global model synchronization. Things can only worsen when considering the heterogeneity of communication environments across different devices, where the learning efficiency may severely depend on a few stragglers with poor channel conditions.

Several pioneering works have been done to manage system resources for efficient FL in wireless networks~\cite{wang2019adaptive, yang2020federated, chen2020convergence, tran2019federated, dinh2019federated}. However, these studies overlooked reducing resource consumption intrinsically from the learning algorithm's perspective, hindering a substantial performance boost in resource utilization and training efficiency. A promising solution suggested in recent works in distributed learning is to incorporate state-of-the-art communication compression strategies into FL algorithms, which can considerably reduce the communication cost with little impacts on learning outcomes~\cite{yu2019parallel, stich2018sparsified,alistarh2018convergence, alistarh2017qsgd,ding2021differentially,tang2018communication}. Yet the existing compressed distributed learning algorithms and the corresponding convergence analysis typically require identical compressor across all the participants, which ignores the heterogeneity in participants' communication capacity and thereby exhibits less flexibility. More importantly, this line of works mainly focuses on alleviating communication burdens in FL. However, the ever-increasing deployment of 5G networks that provides data rates as high as 1Gbps has shown a great potential to eliminate the communication bottleneck in FL, let alone the forthcoming 6G revolution~\cite{saad2019vision}. For example, to transmit a Resnet-50, a commonly used deep network for image classification, with approximately 100MB parameters via 1Gbps wireless links typically consumes 0.16J, which is comparable to the energy consumption of performing a single-step local training on one GPU (e.g., 0.2J for NVIDIA Tesla V100~\cite{lin2020don,goyal2017accurate}). In light of this fact, it is worthwhile to investigate the impacts of both ``working'' (i.e., local computing) and ``talking'' (i.e., wireless transmission) and strike a balance between them via flexible compression control.

In designing a compression control scheme, we would like to communicate as few times and bits as possible to reduce the communication cost. At the same time, we attempt to incur as little distortion to the gradient information as possible towards fast convergence. However, these two goals are fundamentally in conflicts since deeper compression will naturally lead to more distortion on the gradients and more potential communication rounds to converge. When taking energy consumption of edge devices as the measure, this can be further interpreted as follows: Compressing the gradients severely and performing single-step local updates all the time could minimize the energy consumption per communication round. This may require extra communications to attain the targeted model accuracy, or even make the model fail to converge, and thereby impair the overall energy efficiency. 

To tackle the challenges above, in this work, we study to improve the energy efficiency of FL over heterogeneous mobile edge devices. Considering the heterogeneous environments across participating wireless edge devices, we propose a flexibly compressed learning algorithm integrating local computation, gradient sparsification, error compensation, and batch size increment. Based on the convergence rate of the algorithm, we develop a compression control scheme that adapts the compression parameters to minimize all the devices' energy consumption on computing and communication. Our salient contributions are summarized as follows: 

\begin{itemize}
\item  We propose a convergence-guaranteed FL algorithm enabling flexible communication compression, which allows participants to compress the gradients to different levels before uploading. The convergence rate is analyzed theoretically and some insightful results are highlighted. 

\item From the long-term learning perspective, we formulate a compression control problem using the derived convergence bound, where the goal is to achieve energy efficient federated training on edge GPUs over wireless networks. 

\item Capturing the heterogeneity of participating edge devices in their computing and communication environments, we develop a control algorithm integrating Benders decomposition and inner convex approximation to determine the compression parameters for each participant. 


\item We evaluate the performance of the proposed control scheme via extensive simulations, which verify the efficacy of our algorithms with various data sources, learning architectures, and system configurations. 
\end{itemize}

The remainder of the paper is organized as follows. Section II reviews related work. Section III elaborates on the flexibly compressed FL procedures and provides the convergence analysis. Section IV presents an energy-efficient compression control algorithm. Section V gives the performance evaluation, and Section VI finally concludes the paper.

\section{Related Work}
FL over wireless networks has recently gained tremendous attention, whose system design is entangled with training acceleration, network optimization, and on-device resource allocation. Recognizing the limited computing and communication resources at edge computing systems, Wang et al. in \cite{wang2019adaptive} dynamically controlled the frequency of global synchronization to minimize the learning loss in real-time adapting to the resource budget. By exploring the unique properties of wireless multiple-access channels, Yang et al. in \cite{yang2020federated} developed a fast model aggregation approach with joint device selection and beamforming design, which considers only one communication round, and thus cannot guarantee the long-term training performance. To accelerate the training process, Chen et al. in \cite{chen2020convergence} scheduled the participants of high significance for model uploading per communication round while allocating the uplink wireless resources properly. Capturing the trade-off between training time and participants’ energy consumption, authors in \cite{tran2019federated} and \cite{dinh2019federated} formulated optimization problems to jointly allocate the computing and communication resources by considering the heterogeneity of environments. While properly managing the system resources to enable FL in mobile edge networks, these studies overlook reducing resource consumption intrinsically in the essence of learning algorithm itself, thus hindering the substantial boost in training efficiency and resource utilization. 

Some recent efforts on distributed learning algorithm design have been devoted to mitigating the communication bottleneck, which can be categorized into two directions: communication round reduction and communication traffic reduction. Specifically, McMahan et al. in \cite{googleFL} proposed the FedAvg algorithm (also known as local SGD) to reduce the frequency of global synchronization, which allows every participant to perform multiple local SGD iterations in a communication round, other than communicating after every local iteration. Authors in \cite{smith2017don} and \cite{ yu2019computation} used dynamically increasing batch sizes in distributed SGD to reduce the required number of communication rounds. In \cite{yu2019linear}, a momentum SGD method was adopted to accelerate the convergence where the involved communications during training can be reduced accordingly. To reduce the traffic per communication round, one could let each participant communicate the compressed gradients rather than raw gradients for every global synchronization. For example, sparsified SGD studied in \cite{ stich2018sparsified, alistarh2018convergence } followed the idea that only a small subset of gradients with large magnitude are required to upload. Quantized SGD studied in \cite{ alistarh2017qsgd,tang2018communication} allowed each participant to quantize the gradients into low-precision values before sending. Despite reduced communication complexity, the huge computing cost remains hinder FL on resource-constrained edge devices. Besides, most of the current gradient sparsification methods require identical sparsity levels across all the participants, ignoring the heterogeneity of participants and thereby exhibiting less flexibility. This work fills this gap by redesigning the compressed-federated learning algorithm with flexible and well-controlled compression parameters (e.g., global synchronization frequency and gradient sparsity). Simultaneously, heterogeneous computing and communication environments of participants are jointly considered to make the compression strategy suitable for mobile edge networks in practice. 

\section{Federated Learning with Compression }

\subsection{Federated Learning Algorithm Design}
We consider an edge computing powered wireless network in which one base station and a set of mobile participants, denoted by $\mathcal{M}=\{1,2,...,m,...,M\}$, collaboratively train a deep neural network model via FL. We follow the common settings of synchronous FL as in \cite{googleFL} and assume that each participant maintains a locally collected dataset. The goal of collaborative training is to learn a global model that achieves uniformly good performance over all the participants, which can be formally described as minimizing a finite-sum non-convex objective $F: \mathbb{R}^d \rightarrow \mathbb{R}$ of the form
\begin{equation}
    F(\boldsymbol {w}) \triangleq  \frac{1}{M}\sum \limits_{m=1}^M f_m(\boldsymbol{w}).
\end{equation}

Here, $f_m(\boldsymbol {w})$ is the loss function defined by the participant $m$’s local dataset $\mathcal{D}_m $ and the parameter vector $\boldsymbol {w}$. Specifically,
\begin{equation}
    f_m(\boldsymbol {w}) = \frac{1}{|\mathcal{D}_m|} \sum \limits_{i\in \mathcal{D }_m }f_m^i(\boldsymbol {w};x_m^i,y_m^i),
\end{equation}
where $|\mathcal{D}_m|$ is the size of $\mathcal{D}_m $ and $(x_m^i,y_m^i)$ is the $i$-th sample in $\mathcal{D}_m $. Note here that we usually have $\mathcal{D}_m \neq \mathcal{D}_j $ and $\nabla f_m(\boldsymbol {w}_m) \neq \nabla f_j(\boldsymbol {w}_j)$ for any $m \neq j$ since data are usually heterogeneously-distributed across participants in typical FL applications.

Aiming at reducing the communication cost during FL, we propose the \textbf{F}lexible ${\rm \textbf{T}op}_k$ \textbf{L}ocal \textbf{S}tochastic \textbf{G}radient \textbf{D}escent with \textbf{D}ynamic \textbf{B}atch sizes (\textbf{FT-LSGD-DB}) algorithm by integrating two state-of-the-art communication compression strategies, namely, local computations and gradient sparsification. The former allows each participant to perform more local computations on the edge device between every two global synchronizations, thereby reducing the total number of communication rounds. The latter lets participants explicitly sparsify the updated gradient tensors before uploading by retaining only a fraction of components, thereby reducing the size of communication payload in each round. Here, we use ``${\rm Top}_k$'' compressor, a commonly used gradient sparsification approach, to take the sparsified top-$k$ gradients. Specifically, for a vector $\boldsymbol{x}\in \mathbb{R}^d$, ${\rm Top}_k(\boldsymbol{x})\in \mathbb{R}^d$, and the $i^{th}(i=1,2,...,d)$ element of ${\rm Top}_k(\boldsymbol{x})$ is defined by:
\begin{equation}\label{topk}
{\rm Top}_k(\boldsymbol{x}^{^i})=\left\{
\begin{aligned}
\boldsymbol{x}^{^i}, & & {|\boldsymbol{x}^{^i}| \geq thr},\\
0\ \ , & & {\text{otherwise},}
\end{aligned} \right.
\end{equation}
where $\boldsymbol{x}^{^i}$ denotes the $i^{th}$ element of $\boldsymbol{x}$ and $thr$ is the $k$-th largest absolute value of the elements in $\boldsymbol{x}$. In practice, $k$ can be two to three orders of magnitude smaller than $d$ while only sacrificing the model accuracy to a mild extent. In this case, the communication overhead involved in gradient transmission can be dramatically saved~\cite{alistarh2018convergence,sattler2019sparse}.

\begin{algorithm}[!t]
\caption{FT-LSGD-DB Algorithm} \label{alg:topkSGD}
\hspace*{0.02in} {\bf Input:} The dataset $\{\mathcal{D}_m\}_{\forall m}$; The initialized mini-batch size: $b^{(0)}$; The mini-batch size scaling factor: $\rho>1$; The number of participants: $M$; The number of iterations to train: $T$\\
\hspace*{0.02in} {\bf Output:} Final model parameter $\boldsymbol{w}^{(T)}$ \\ 
\hspace*{0.02in} {\bf Initialization:} $\boldsymbol{w}^{(0)}=\widehat{\boldsymbol{w}}_m^{(0)}=\boldsymbol{e}_m^{(0)},\ \forall m\in \mathcal{M}$
\begin{algorithmic}[1]
\For{$t=0,1,2,...,T-1$ }
    \State {\bf On Edge Devices:}
    \For{$m\in \mathcal{M}$ in parallel}
        \State $b^{(t)}\leftarrow\lfloor \rho^t b^{(0)} \rfloor$
        \State Sampling a mini-batch $\mathcal{D}_m^{(t)}$ of size $b^{(t)}$ from $\mathcal{D}_m$
        \State $\widehat{\boldsymbol{w}}_m^{(t+\frac{1}{2})} \leftarrow \widehat{\boldsymbol{w}}_m^{(t)} - \eta^{(t)} \nabla f_m(\boldsymbol {w}^{(t)};\mathcal{D}_m^{(t)})$
        \If{$t+1$ is an integer multiple of $H$}
            \State $\boldsymbol{u}_m^{(t)}\leftarrow \boldsymbol{e}_m^{(t)}+ \boldsymbol{w}^{(t)}-\widehat{\boldsymbol{w}}_m^{(t+\frac{1}{2})}$ 
            \State $\boldsymbol{g}_m^{(t)}= {\rm Top}_{k_m}(\boldsymbol{u}_m^{(t)})$ and upload $\boldsymbol{g}_m^{(t)}$
            \State $\boldsymbol{e}_m^{(t+1)} \leftarrow \boldsymbol{e}_m^{(t)}+\boldsymbol{w}^{(t)}-\widehat{\boldsymbol{w}}_m^{(t+\frac{1}{2})}-\boldsymbol{g}_m^{(t)}$ 
            \State Receive $\boldsymbol{w}^{(t+1)}$ and  $\widehat{\boldsymbol{w}}_m^{(t+1)} \leftarrow \boldsymbol{w}^{(t+1)}$  
        \Else
            \State $\widehat{\boldsymbol{w}}_m^{(t+1)}\leftarrow \widehat{\boldsymbol{w}}_m^{(t+\frac{1}{2})}$
            \State $\boldsymbol{w}^{(t+1)}\leftarrow \boldsymbol{w}^{(t)}$
            \State $\boldsymbol{e}^{(t+1)}\leftarrow \boldsymbol{e}^{(t)}$
        \EndIf
    \EndFor
    \State {\bf At Central Server:}
    \If{$t+1$ is an integer multiple of $H$}
        \State Collect $\boldsymbol{g}_m^{(t)}, \forall m$ and $\boldsymbol{g}^{(t)}= \frac{1}{M}\sum_{m=1}^M  \boldsymbol{g}_m^{(t)}$
        \State $\boldsymbol{w}^{(t+1)}=\boldsymbol{w}^{(t)}- \boldsymbol{g}^{(t)}$ and broadcast $\boldsymbol{w}^{(t+1)}$
    \Else
        \State $\boldsymbol{w}^{(t+1)}\leftarrow \boldsymbol{w}^{(t)}$
    \EndIf
\EndFor
\State \Return $\boldsymbol{w}^{(T)}$
\end{algorithmic}
\end{algorithm}
Unlike previous studies requiring identical sparsity level across all the participants, FT-LSGD-DB injects more flexibility into training procedures by allowing the participating devices to perform gradient sparsification with different values of ``$k$''. This indeed helps to accommodate stragglers with poor channel conditions and thus mitigates the impacts of stale updates. Besides, FT-LSGD-DB novelly incorporates error compensation and batch size increment into FL procedures, which are two effective methods adopted and verified by practitioners recently. Specifically, error compensation is used to accelerate the global convergence by accumulating the error that arises from only uploading sparse approximations of the gradient updates, which ensures all gradient information does get eventually aggregated~\cite{stich2018sparsified}. Using dynamically increasing batch sizes during training can maintain the known convergence rate with fewer communication rounds~\cite{yu2019computation}. The pseudocode of our FT-LSGD-DB algorithm is given in Algorithm \ref{alg:topkSGD}, and the details are described in the following.

Let $\{1,2,...,t,...T\}$ denote a set of iteration indices and assume that each participant performs $H$ steps of local updates between every two global synchronizations. In each iteration $t$, evey participant $m \in \mathcal{M}$ performs:
\begin{enumerate}
    \item \textit{Batch size increment:} Exponentially increases its own SGD batch size with a factor $\rho$.
    \item \textit{Local update:} Update local parameter $\boldsymbol {w}_m$ using the stochastic gradient $\nabla f_m(\boldsymbol {w};\mathcal{D}_m^{(t)})$, where $\mathcal{D}_m^{(t)}$ is a mini-batch of size $b$ sampled uniformly from $\mathcal{D}_m$ at the $t$-th iteration.
\end{enumerate}
	 
If aggregation is performed at iteration $t$ (i.e., $t$ is an integer multiple of $H$), every participant $m \in \mathcal{M}$ performs: 
\begin{enumerate}
    \setcounter{enumi}{2}
    \item \textit{Error compensation:} Add the local error $\boldsymbol {e}_m^{(t)}$ from the previous iteration into the gradient $\boldsymbol {g}_m^{(t)}$.
    \item \textit{Gradient sparsification:} Truncate the gradient sum to its top $k_m$ components, sorted in decreasing order of absolute magnitude.
    \item \textit{Gradient upload:} Send the sparsified error-compensated gradient $\boldsymbol {g}_m^{(t)}$ to the base station. 
    \item \textit{Error accumulation:} Update the local error $\boldsymbol {e}_m^{(t)}$.
\end{enumerate}

Upon receiving $\boldsymbol {g}_m^{(t)}$ from all the participants, the base station aggregates them, updates the global model, and broadcasts the new model $\boldsymbol {w}^{(t+1)}$ to participants. Every participant $m\!\in\!\mathcal{M}$ set its local parameter $\boldsymbol {w}^{(t+1)}_m$ to be equal to the global parameter $\boldsymbol {w}^{(t+1)}$. The training process above is repeated until achieving satisfactory accuracy. The following section further shows the convergence rate achieved by Algorithm \ref{alg:topkSGD} and derives the corresponding communication complexity.

\subsection{Convergence Analysis and Discussion}
We consider the following two standard assumptions on the local loss functions $f_m: \mathbb{R}^d \rightarrow \mathbb{R}, \forall m \in \mathcal{M}$.
\begin{assumption}[Smoothness]\label{asp:1}
 $f_m(\cdot)$ is $L$-smooth, i.e., for every $\boldsymbol{w},\boldsymbol{w}'\in \mathbb{R}^d$, we have
    \begin{equation}
        f_m(\boldsymbol{w})\leq f_m(\boldsymbol{w}')+<\nabla f_m(\boldsymbol{w}), \boldsymbol{w}'-\boldsymbol{w}> + \frac{L}{2}||\boldsymbol{w}'-\boldsymbol{w}||^2.
    \end{equation}
\end{assumption}

\begin{assumption}[Bounded variances and second momentum]\label{asp:2}
For every $\boldsymbol{w}_m^{(t)} \in \mathbb{R}^d$ and $t \in \mathbb{Z}^+$, there exists constants $\sigma > 0$ and $G \geq \sigma$ such that:   
    \begin{equation}
        \mathbb{E} _{\mathcal{D}_m^{(t)}\subset \mathcal{D}_m} [||\nabla f_{m}(\boldsymbol{w}_m^{(t)};\mathcal{D}_m^{(t)})\!-\!\nabla f_{m}(\boldsymbol{w}_m^{(t)}||^2] \leq \sigma^2,\ \forall m,
    \end{equation}
    \begin{equation}
        \mathbb{E} _{\mathcal{D}_m^{(t)}\subset \mathcal{D}_m} [||\nabla f_{m}(\boldsymbol{w}_m^{(t)};\mathcal{D}_m^{(t)}) ||^2] \leq G^2,\ \forall m.
    \end{equation}
\end{assumption}

Let $\delta_m=d/k_m \geq 1$ be the gradient sparsity chosen by the $m$-th participant. Under the assumptions above, the following theorem hold when Algorithm \ref{alg:topkSGD} is run with the sparsity series $\{\delta_m\}_{\forall m}$.

\begin{theorem}\label{theorem1}
Suppose a constant learning rate $\eta_t=\eta=\frac{\theta\sqrt{M}}{\sqrt{T}},\forall t\geq 0$ is chosen where $\theta >0$ is a constant satisfying $\frac{\theta\sqrt{M}}{\sqrt{T}}\leq \frac{1}{2L}$, we have the convergence rate for Algorithm \ref{alg:topkSGD}:
\begin{equation}  \label{ConvergenceResult}
\begin{split}
&\mathbb{E}[||\boldsymbol{z}_T||^2] \\
\leq &\frac{4(\mathbb{E}[F(\boldsymbol{w}^{(0)})]-F^*)}{\theta \sqrt{MT}}\\
&+\frac{8\rho\theta L \sigma^2}{(\rho-1)b^{(0)}\sqrt{M} T^{3/2}}
+(4\delta^2+1)\frac{8M\theta^2L^2G^2H^2}{T},
\end{split}
\end{equation}
where $\delta=\sqrt{\frac{1}{M}\sum\limits_{m=1}^M \delta_m^2}$ is the root mean square of the sparsity series $\{\delta_m\}_{\forall m}$ and $\boldsymbol{z}_T$ is a random variable which samples a previous parameter $\widehat{\boldsymbol{w}}_m^{(t)}$ with probability $1/MT$.
\end{theorem}

\begin{proof}
Please refer to Appendix for the proof.
\end{proof}

In Theorem \ref{theorem1}, we settle for a weaker notion of convergence and use the average expected squared gradient norm to characterize the convergence rate due to the non-convex settings as \cite{alistarh2018convergence} does. Based on this, we further give the following corollary on communication complexity of our FT-LSGD-DB algorithm.

\begin{corollary}\label{corollary1}
Let $\rho=\frac{T}{T-1}$, $\theta^2=\frac{b^{(0)}(\mathbb{E}[F(\boldsymbol{w}^{(0)})]-F^*)}{\sigma^2L}$ and $\mathbb{E}[F(\boldsymbol{w}^{(0)})]-F^*\leq J^2$ where $J\leq \infty$ is a constant. The maximum number of global communication rounds required for achieving an $\varepsilon$-global model convergence, i.e., satisfying $\mathbb{E}[||\boldsymbol{z}_T||^2] \leq \varepsilon$, is given by
    \begin{equation}\label{GlobRound}
    \begin{split}
    K(\delta,H)= &\mathcal{O}\left(MH\delta^2\right)+\mathcal{O}\left(\frac{1}{\sqrt{M}H}\right).
    \end{split}
    \end{equation}
\end{corollary}

\begin{proof}
Substituting $\rho=\frac{T}{T-1}$, $\theta^2=\frac{b^{(0)}(\mathbb{E}[F(\boldsymbol{w}^{(0)})]-F^*)}{\sigma^2L}$ and $\mathbb{E}[F(\boldsymbol{w}^{(0)})]-F^*\leq J^2$ into (\ref{ConvergenceResult}) yields
\begin{equation}  
\begin{split}
\mathbb{E}[||\boldsymbol{z}_T||^2] \leq \frac{12\sigma J\sqrt{L}}{\sqrt{MTb^{(0)}}}+(4\delta^2+1)\frac{8b^{(0)}MLG^2H^2J^2}{\sigma^2T}.
\end{split}
\end{equation}

According to the convergence criterion, we suppose that
\begin{equation}  
\begin{split}
\varepsilon =&\frac{12\sigma J\sqrt{L}}{\sqrt{MTb^{(0)}}}+(4\delta^2+1)\frac{8b^{(0)}MLG^2H^2J^2}{\sigma^2T}.
\end{split}
\end{equation}

Rearranging the terms, we get the maximum number of iterations as follows:
\begin{flalign}
T(\delta,H)=&\frac{8b^{(0)}MLG^2J^2H^2(4\delta^2+1)}{\varepsilon \sigma^2}+\frac{72L\sigma^2J^2}{\varepsilon^2b^{(0)}M}\\
&+\sqrt{\frac{8b^{(0)}MLG^2J^2H^2(4\delta^2+1)}{\varepsilon \sigma^2}+\frac{36L\sigma^2J^2}{\varepsilon^2b^{(0)}M}}.\nonumber
\end{flalign}

In Algorithm \ref{alg:topkSGD}, communications are only needed to aggregate individual gradient-update and happen only once every $H$ iterations. Hence, the total number of necessary communication rounds is given by $K=T/H$, i.e.,
\begin{flalign}
    K(\delta,H)=&\frac{8b^{(0)}MLG^2J^2H(4\delta^2+1)}{\varepsilon \sigma^2}+\frac{72L\sigma^2J^2}{\varepsilon^2b^{(0)}MH} \nonumber\\
&+\sqrt{\frac{8b^{(0)}MLG^2J^2(4\delta^2+1)}{\varepsilon \sigma^2}+\frac{36L\sigma^2J^2}{\varepsilon^2b^{(0)}MH^2}} \nonumber\\
=&\mathcal{O}\left(MH\delta^2\right)+\mathcal{O}\left(\frac{1}{\sqrt{M}H}\right).
    \end{flalign}
\end{proof}

The results in (\ref{ConvergenceResult}) and (\ref{GlobRound}) indicate that the gradient sparsity magnitudes of all the participants jointly take impacts on global convergence and communication complexity. Given a target model accuracy (i.e., $\varepsilon$), a higher $\delta$ results in a larger bound of communication rounds. Besides, aggressively enlarging $H$ (i.e., ``working'' more) can also impair the learning efficiency as more communications may be involved. 

\section{Energy-efficient Federated Learning on GPUs: Problem Formulation and Control Algorithm}
The theoretical results above reveal that both gradient sparsity levels $\{\delta_m\}_{\forall m}$ and global update frequency $H$ play critical roles in convergence rate and communication efficiency from the learning perspective. Considering a realistic edge computing environment, we highlight that $\{\delta_m\}_{\forall m}$ and $H$ also have great impacts on the energy consumption of participating edge devices, because they affect the payload required for transmitting and the workload required for processing, respectively. In this section, we aim to tune these two types of compression parameters accommodating heterogeneous FL participants for optimizing overall energy efficiency. 

\subsection{System Model and Problem Formulation}

\subsubsection{Communication model} 
Let $S_m$ denote the total number of bits communicated by the $m$-th participant per global round. Using the ``${\rm Top}_k$’’ compressor defined in (\ref{topk}), one needs to send the values and the positions of the non-zero gradients in the flattened tensors after sparsification. Let ${\rm FPP}$ denote the floating-point precision, e.g., ${\rm FPP}=32$ for single-precision floating-points and ${\rm FPP}=64$ for double-precision floating-points. With the sparsity $\delta_m=d/k_m$, participant $m$ needs ${\rm FPP}$ bits to represent the absolute value of each non-zero gradient with one extra bit indicating its sign, i.e., 
\begin{equation}
    S_{m,val}=({\rm FPP}+1)\times k_m  \ \ \text{bits}.
\end{equation}

The positions of the non-zero entries can be identified by enumerating all possible sparsity patterns, which require
\begin{equation}\label{position}
    S_{m,pos}=\log_2 \binom{d}{k_m} \ \ \text{bits}
\end{equation}
to represent. Accordingly, we define $S_m$ as 
\begin{equation}\label{DataSize}
    S_m=s_1 (S_{m,val}+ S_{m,pos}) + s_0 \ \ \text{bits},
\end{equation}
where $s_0$ and $s_1$ are coefficients indicating extra communication overhead involved in wireless transmitting~\cite{khirirat2020communication}. Note that a federated training task usually lasts for a time duration in tens of minutes due to the huge volume of data required for transferring as well as the high computational complexity in running SGD. Thus, the channel conditions of participants may suffer from great fluctuations during a training period. For this reason, it is expected to consider the energy consumption of a training task from a long-term learning perspective. Here, we employ the average transmission rate of every participant $m \in \mathcal{M}$, which is evaluated by

\begin{equation}
    R_m = W_m\mathbb{E}_{h_m}[\log_2 (1+\frac{P_m |h_m|^2}{N_0})],
\end{equation}
where the expectation is taken over channel fading $h_m$ between participant $m$ and the base station; $N_0$ indicates the power of additive white Gaussian noise; $W_m$ and $P_m$ denote the bandwidth and the transmitting power of participant $m$, respectively \cite{pan2011joint}. Afterward, the energy consumed to transmit the sparsified gradients by participant $m$ is calculated as
\begin{equation}
    E_m^{com}=\frac{P_m S_m}{R_m}.
\end{equation}

\subsubsection{Computational model}
On-device learning, especially for training deep network models, is a compute-intensive task that has proved challenging to achieve adequate performance when running merely on CPUs of commodity mobile devices. Fueled by the recent advances in mobile hardware technology, GPU has become a ubiquitous hardware accelerator integrated virtually in every smart device to offer significantly more compute power. A typical GPU chip includes a multi-core GPU module and an associated GPU memory module where the voltage and frequency of GPU cores and GPU memory can be controlled separately. We model the energy consumed to execute a single iteration of GPU-accelerated mini-batch SGD at the $m$-th edge device as the product of the runtime power and the execution time, i.e.,
\begin{equation}\label{ite_ener}
    E_{m,ite}^{cmp}= P_m^{cmp} \cdot T_m^{cmp},
\end{equation}
where $P_m^{cmp}$ and $T_m^{cmp}$ are two functions of the core voltage and the core/memory frequency~\cite{mei2017energy}, which are given by
\begin{equation}
    P_m^{cmp}=P_m^0+\alpha f_m^{mem}+ \beta (v_m^{core})^2 f_m^{core},
\end{equation}
\begin{equation}
    T_m^{cmp}=T_m^0+\frac{a}{f_m^{mem}} + \frac{b}{f_m^{core}}.
\end{equation}

Here, $P_m^0$ and $T_m^0$ represent the static power consumption and static time consumption; $v_m^{core},f_m^{core},f_m^{mem}$ denote the GPU core voltage, GPU core frequency, and GPU memory frequency, respectively; $\alpha$, $\beta$, $a$ and $b$ are constant coefficients indicating the sensitivity to memory frequency scaling and the core voltage/frequency scaling, which depend on the hardware and the application characteristics. In this work, the value of $\alpha$, $\beta$, $a$, and $b$ are derived from platform-based experiments by measuring the average runtime energy consumption. Specifically, we measure the energy consumed to execute single-step SGD and estimate the parameters that appeared in the energy model in (\ref{ite_ener}). For simplicity, we shall assume that $E_{m,ite}^{cmp}$ keeps unchanged during training in spite of the incremental batch sizes used in Algorithm \ref{alg:topkSGD}. This is reasonable due to the fact that GPUs are capable of parallel execution. When the training batch size remains under a threshold, GPUs can process the whole-batch samples simultaneously, leading to a near-constant execution time \cite{ren2019accelerating,lin2020don}. In this case, the total energy\footnote{The computational complexity of the gradient sparsification algorithm is so low compared with running local SGD that the corresponding computational workload and the involved energy consumption can be omitted \cite{ren2019accelerating}.} consumed between every two global synchronizations including $H$ local iterations can be computed as
\begin{equation}
    E_m^{cmp}=E_{m,ite}^{cmp} \cdot H.
\end{equation}

\subsubsection{Problem Formulation}
Given the communication model and the computational model above, we compute the total energy consumption of all the participating edge devices between every two global synchronizations as
\begin{equation}
    E=\sum \limits_{m=1}^M\left( E_m^{com} + E_m^{cmp}\right),
\end{equation}
which captures the heterogeneity of participants on their communication conditions and GPU capacities. Exploiting Corollary \ref{corollary1} and plugging $\delta=\sqrt{\frac{1}{M}\sum_{m=1}^M \delta_m^2}$, we model the overall energy consumed during the whole training process as 
\begin{equation}\label{obj}
    \Gamma(\delta_1,\delta_2,...,\delta_M,H)= E\cdot\sum \limits_{m=1}^M\left(\alpha H\delta_m^2+\frac{\beta}{M^{3/2}H}\right),
\end{equation}
where $\alpha$ and $\beta$ are constants used to approximate the big-$\mathcal{O}$ notion in (\ref{GlobRound}). With the goal of overall energy consumption minimization, we jointly determine the gradient sparsity $\delta_m$ for each participant $m\in\mathcal{M}$ and the global update frequency $H$ by solving the following optimization problem: 
\begin{subequations}  \label{ProbForm}
\begin{align}
\min \limits_{\{\delta_m,H\}} \quad &  \Gamma(\delta_1,\delta_2,...,\delta_M,H) \\
s.t.\quad & \delta_{lb}\leq \delta_m \leq \delta_{ub}, \ \forall m, \label{c1}\\
& H \in \mathcal{H}. \label{c2}
\end{align}
\end{subequations}

Here, constraints (\ref{c1}) and (\ref{c2}) restrict the feasible range of $\delta_m$ and $H$ with $\mathcal{H}\subset \mathbb{Z}^+$, respectively. The formulated problem in (\ref{ProbForm}) exhibits a certain trade-off between compression and convergence in the considered communication-compressed FL setting. To minimize the energy consumption in single global iteration, one will severely compress the gradient tensor and decide to perform a single-step local update all the time, which would greatly impact the convergence rate and increase the number of global synchronizations. As a result, total energy consumption may increase considerably. In practice, the participants with excellent communication environments are expected to adopt slight gradient compression schemes for accelerating the convergence, while the others with poor channel conditions should be allowed to sparsify the gradients more severely to save the energy. In light of this, global update frequency and gradient sparsity should be carefully determined by considering the heterogeneity of participants for achieving energy-efficient FL.

\subsection{Compression Control Algorithm}
We develop a compression parameter control algorithm by approximately solving the formulated optimization problem in (\ref{ProbForm}) that falls into the category of mixed-integer non-linear programming. It is non-trivial to solve since the integer variable $H$ is highly coupled with the continuous variables $\{\delta_m\}_{\forall m}$. In the following, we first transform the permutation operator in (\ref{DataSize}) into a tractable form. Then we propose an efficient algorithm integrating generalized Benders decomposition and inner convex approximation to find a satisfactory solution of the considered compression control problem.  

\begin{corollary}\label{proposition1}
Let $\delta_m\!=\!d/k_m$ be the chosen gradient sparsity and $\kappa\!=\!{\rm FPP}\!+\!1$. The total number of bits required to be transmitted by participant $m$ per global round, i.e., $S_m$, can be approximated to 
\begin{equation}\label{Sval}
    S_m(\delta_m)=\frac{s_1d}{\delta_m}(\log_2 \delta_m+\kappa) + s_0
\end{equation}

\end{corollary}
\begin{proof}
\vspace{-0.05in}
\begin{equation}
\begin{aligned}
&S_{m,pos}=\log_2 \binom{d}{k_m} 
=  \log_2 \frac{d!}{(d-k_m)!k_m!}\\
=& \log_2 d!-\log_2 k_m!-\log_2(d-k_m)! \\
\overset{(a)}\approx & d\log_2 d-k_m \log_2 k-(d-k_m)\log_2 (d-k_m) \\
=& d\log_2 d-\frac{d}{\delta_m}\log_2 \frac{d}{\delta_m}-(d-\frac{d}{\delta_m})\log_2 (d-\frac{d}{\delta_m}) \\
=&d [\log_2 \delta_m +(\frac{1}{\delta_m}-1)\log_2(\delta_m-1)] 
\overset{(b)}\approx \frac{d}{\delta_m} \log_2 \delta_m. \nonumber
\end{aligned}
\end{equation}
Here, $(a)$ is by Stirling formula that gives precise estimate for factorials, i.e., $\log_2{n!}\!\approx\! n\log_2{n}\!-\!n\log_2{e}$, and $(b)$ is due to $\delta_m\!=\!d/k_m\!\gg\!1$. Substituting $S_{m,pos}$ into (\ref{DataSize}) yields (\ref{Sval}).
\end{proof}

We can easily verify that $S_m(\delta_m)$ is strongly convex w.r.t. $\delta_m$, when $\delta_m\geq e^{3/2}$ by calculating its second order derivative:
\begin{equation}
\begin{split}
    \frac{\mathrm{d}^2 S_m(\delta_m)}{\mathrm{d}\delta_m^2}=\frac{s_1d(2\ln \delta_m -3)}{\delta_m^3\ln 2}+\frac{2s_1d \kappa}{\delta_m^3 }>0.
\end{split}
\end{equation}

Extensive empirical evidence reveals that $\delta_m\!\geq\! e^{3/2}$ always holds in practice so that gradient sparsification can considerably reduce the communication overhead. In the following, we assume that $\delta_{lb}$ is set to be no less than $e^{3/2}$ and view $S_m(\delta_m)$ as a strongly convex function without extra conditions. Using Corollary \ref{proposition1}, we substitute $\Gamma(\cdot)$ by an approximated energy cost function and rewrite the problem in (\ref{ProbForm}) as:
\begin{flalign}\label{ProbForm2}
\min \limits_{\{\delta_m,H\}} \quad &  \sum \limits_{m=1}^M\left(\alpha H\delta_m^2+\frac{\beta}{M^{3/2}H}\right)\nonumber\\ 
    &\cdot \sum \limits_{m=1}^M\left(\frac{P_m s_1d (\log_2 \delta_m+\kappa) }{R_m \delta_m}+\frac{P_m s_0}{R_m} + E_m^{0}H\right) \nonumber\\
s.t.\quad    & (\ref{c1})\  \text{and}\  (\ref{c2}). 
\end{flalign}

To solve (\ref{ProbForm2}), we propose an algorithm integrating generalized Benders decomposition and inner convex approximation. Specifically, generalized Benders decomposition performs as the outer-loop algorithm to decompose the problem above into two sub-problems: a primal problem w.r.t. the continuous variables $\{\delta_m\}_{\forall m}$ and a master problem w.r.t. the integer variable $H$~\cite{geoffrion1972generalized}. As the inner-loop algorithm, inner convex approximation is used to solve the primal problem by successively optimizing the approximants of the non-convex objective. We solve the primal problem and the master problem in an alternative and iterative manner, as detailed in Algorithm~\ref{algBD}. In each outer-loop iteration, solving the primary problem with given $H$ yields an upper bound for the optimal value of (\ref{ProbForm2}) while solving the master problem provides its lower bound. Particularly, we formulate the \textit{primal problem} in the $i$-th iteration with fixed $H^{(i)}$ as follows:
\begin{subequations}  \label{PriProb}
\begin{align}
\min \limits_{\{\delta_m\}} \quad &  \sum \limits_{m=1}^M\left(\alpha H^{(i)}\delta_m^2+\frac{\beta}{M^{3/2}H^{(i)}}\right) \label{PriObj}\\ 
    &\cdot \sum \limits_{m=1}^M\left(\frac{P_m s_1d (\log_2 \delta_m+\kappa) }{R_m \delta_m}+\frac{P_m s_0}{R_m} + E_m^{0}H^{(i)}\right) \nonumber\\
s.t.\quad    & (\ref{c1}). \nonumber
\end{align}
\end{subequations}

Note that the primal problem above is always feasible for all $H^{(i)}$ since the continuous variables $\{\delta_m\}_{\forall m}$ are independent of $H^{(i)}$ in the constraint. Thus, we do not need to check the feasibility of the current $H^{(i)}$ as conventional Benders decomposition methods do~\cite{li2020energy}. Let $\{\lambda_{m1}\}_{\forall m}$ and $\{\lambda_{m2}\}_{\forall m}$ denote two sets of Lagrange multiplier corresponding to the constraints in (\ref{c1}). We solve the primal problem to get the solutions of $\{\delta_m\}_{\forall m}$, $\{\lambda_{m1}\}_{\forall m}$ and $\{\lambda_{m2}\}_{\forall m}$, which are denoted by $\{\delta_m^{(i)}\}_{\forall m}$, $\{\lambda_{m1}^{(i)}\}_{\forall m}$ and $\{\lambda_{m2}^{(i)}\}_{\forall m}$, respectively. We also update $UBD$ with the objective value of (\ref{PriObj}). Afterwards, a feasibility cut can be generated and added to the master problem as a new constraint. In particular, the \textit{master problem} in the $i$-th iteration is given as follows:
\begin{subequations}  \label{MasProb}
\begin{align}
\min \limits_{H} \quad &  \eta \\
s.t.\quad & \eta \geq \sum \limits_{m=1}^M\left(\alpha H(\delta_m^{(l)})^2+\frac{\beta}{M^{3/2}H}\right)\label{FeaCut}\\ 
    &\cdot \sum \limits_{m=1}^M\left(\frac{P_m s_1d (\log_2 \delta_m^{(l)}+\kappa) }{R_m \delta_m^{(l)}}+\frac{P_m s_0}{R_m} + E_m^{0}H\right) \nonumber\\
    &+\!\sum \limits_{m=1}^M \!\left(\lambda_{m1}^{(l)}\! (\delta_{lb}\!-\!\delta_m^{(l)})+\!\lambda_{m2}^{(l)} (\delta_m^{(l)}\!-\!\delta_{ub})\right),
    \forall l\!=1,...,i, \nonumber\\
& H \in \mathcal{H}.  
\end{align}
\end{subequations}

\begin{algorithm}[!t]
\caption{Compression Control Algorithm} \label{algBD}
\hspace*{0.02in} {\bf Initialization:} $H^{(1)}\in \mathcal{H}$; $I_{out}$; $\epsilon=10^{-5}$; $\iota=10^{-5}$; $\xi=10^{-5}$; $LBD=-\infty$; $UBD=\infty$; $i=1$; 
\begin{algorithmic}[1]
\Repeat
    \State Set inner-loop iteration index $\nu=0$
    \State Set step size $\gamma^0\in(0,1]$ and start with $\delta_m^0=\delta_{lb}, \forall m$
    \Repeat
    \State Compute $\delta_m^*(\delta_m^\nu), \forall m$ via (\ref{ClosedForm}) 
    \State Set $\delta_m^{\nu+1}=\delta_m^\nu+\gamma^0(\delta_m^*(\delta_m^\nu)-\delta_m^\nu), \forall m$
    \State Set $\nu=\nu+1$
    \State Set $\gamma^\nu=\gamma^{\nu-1}(1-\xi \gamma^{\nu-1})$
    \Until{$||\boldsymbol{\delta}^\nu-\boldsymbol{\delta}^{\nu-1}||_2^2\leq \iota$}
    \State Save $\delta_m^{(i)}=\delta_m^\nu$ for all $m\in\mathcal{M}$ as the current solution of the primal problem in (\ref{PriProb}).
    \State Obtain the Lagrangian multiplier $\lambda$ and the objective value $\Gamma^{(i)}$ 
    \State Update the upper bound value with $UBD=\Gamma^{(i)}$
    \State Solve the master problem in (\ref{MasProb})
    \State Obtain intermediate solutions of $H^{(i)}$ and $\eta^{(i)}$
    \State Update the upper bound value with $LBD=\eta^{(i)}$
    \State Save $H^{(i+1)}=H^{(i)}$
    \If{$UBD-LBD \leq \epsilon$}
        \State \Return The current solutions of $\{\delta_m\}_{\forall m}$ and $H$.
    \Else{\ Set $i=i+1$}
    \EndIf
\Until{$i=I_{out}$}
\State \Return The current solutions of $\{\delta_m\}_{\forall m}$ and $H$.
\end{algorithmic}
\end{algorithm}

The master problem in (\ref{MasProb}) is a small-scale mixed-integer programming problem and can be solved using classical optimization algorithms, e.g., Branch-and-Bound. After solving it, we update $LBD$ with the value of $\eta$. Now we focus on solving the non-convex primal problem (\ref{PriProb}), which is specified by the inner-loop procedures in Algorithm \ref{algBD}. Here, we use the inner convex approximation method to find the stationary points iteratively. The main idea is to successively optimize certain approximations of the non-convex objective function in (\ref{PriObj}) while maintaining feasibility at each iteration. This requires us to derive a strongly convex approximant of (\ref{PriObj}) around each feasible iteration. With a slight abuse of notation, we denote the objective function in (\ref{PriObj}) by $\Gamma(\boldsymbol{\delta})$ with $\boldsymbol{\delta}\!=\!(\delta_1,...,\delta_M)$ and rewrite it as the product of two functions, i.e., 
\begin{equation}\label{ProConv}
\Gamma(\boldsymbol{\delta})= \Gamma_1(\boldsymbol{\delta}) \cdot \Gamma_2(\boldsymbol{\delta}),
\end{equation}
where  
\begin{flalign} 
    &\Gamma_1(\boldsymbol{\delta})\!=\!\!\sum \limits_{m=1}^M\!\left(\alpha H^{(i)}\delta_m^2+\frac{\beta}{M^{3/2}H^{(i)}}\!\right), \\
    &\Gamma_2(\boldsymbol{\delta})\!=\!\!\sum \limits_{m=1}^M\!\left(\!\frac{P_m s_1d (\log_2\!\delta_m\!\!+\!\kappa)}{R_m \delta_m}\!+\!\frac{P_m s_0}{R_m}\!+\! E_m^{0}H^{(i)}\!\right).
\end{flalign} 

Note that both $\Gamma_1(\boldsymbol{\delta})$ and $ \Gamma_2(\boldsymbol{\delta})$ are positive and strongly convex since non–negative combinations of convex functions preserve convexity. Capturing such the ``product of convexity'' property of (\ref{ProConv}), we built an approximation for $\Gamma(\boldsymbol{\delta})$ as:
\begin{equation} \label{ApproObj} 
\tilde{\Gamma}(\boldsymbol{\delta};\boldsymbol{\delta}^\nu)= \Gamma_1(\boldsymbol{\delta}) \cdot \Gamma_2(\boldsymbol{\delta}^\nu)+\Gamma_1(\boldsymbol{\delta}^\nu)\cdot \Gamma_2(\boldsymbol{\delta}),
\end{equation}
where $\boldsymbol{\delta}^\nu \triangleq (\delta_1^\nu, \delta_2^\nu,...,\delta_M^\nu)$ denote the current intermediate $\boldsymbol{\delta}$ obtained in the $\nu$-th inner iteration~\cite{scutari2016parallel}. Obviously, the approximated objective function in (\ref{ApproObj}) is strongly convex and the corresponding approximations of the primal problem in (\ref{PriProb}) can be solved optimally, which is in the form of:
\begin{equation}  \label{PriProb_Appro}
\begin{split}
\boldsymbol{\delta}^*(\boldsymbol{\delta}^\nu)=\mathop{\text{argmin}} \limits_{\{\delta_m\}_{\forall m}} &\quad   \tilde{\Gamma}(\boldsymbol{\delta};\boldsymbol{\delta}^\nu)\\
s.t.&\quad  (\ref{c1}).
\end{split}
\end{equation}

\begin{figure*} \centering 
\subfigure[Impact of system scale\label{sim1}]
  {\includegraphics[width=4.45cm]{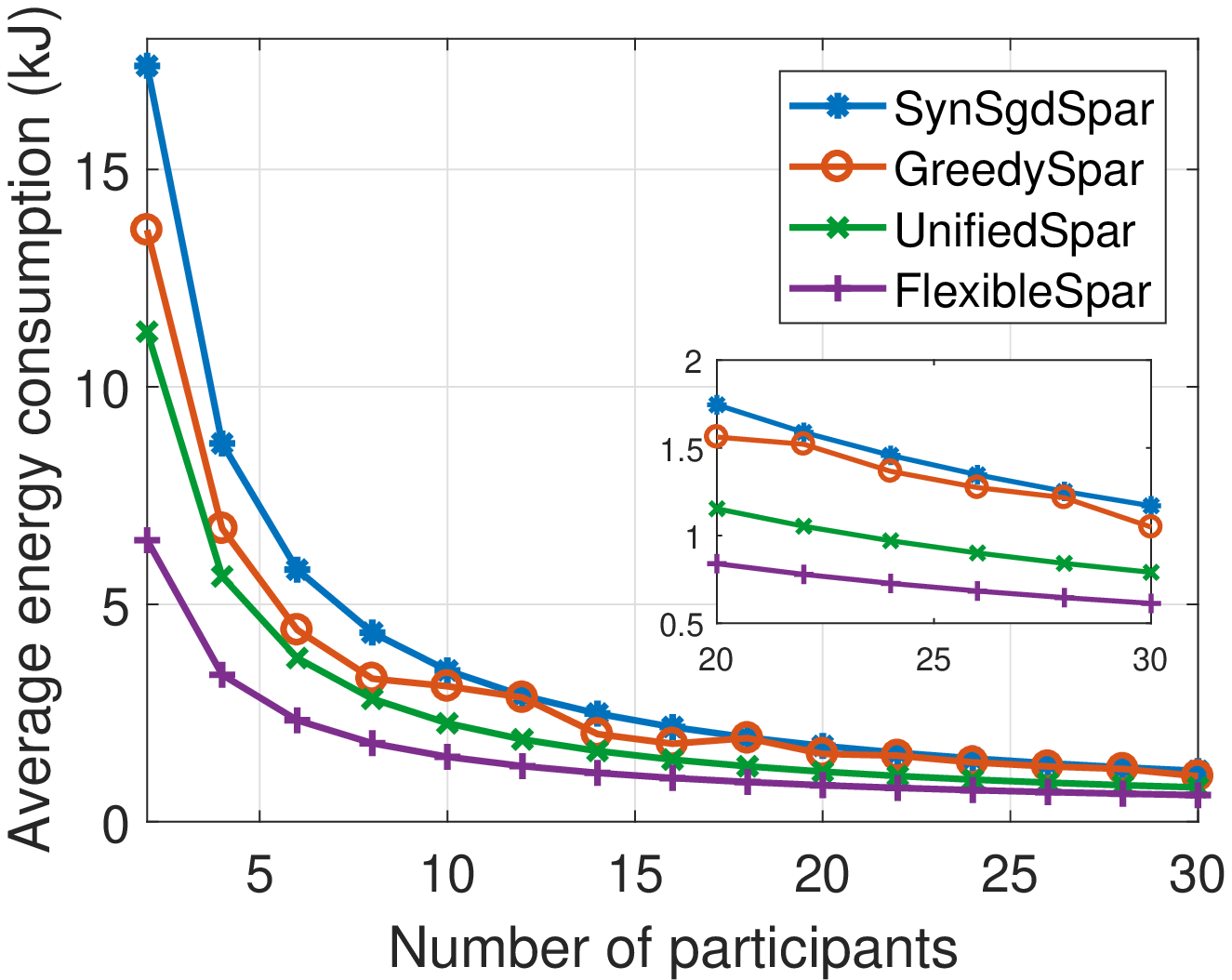}}
\subfigure[Impact of user heterogeneity\label{sim2}]
  {\includegraphics[width=4.45cm]{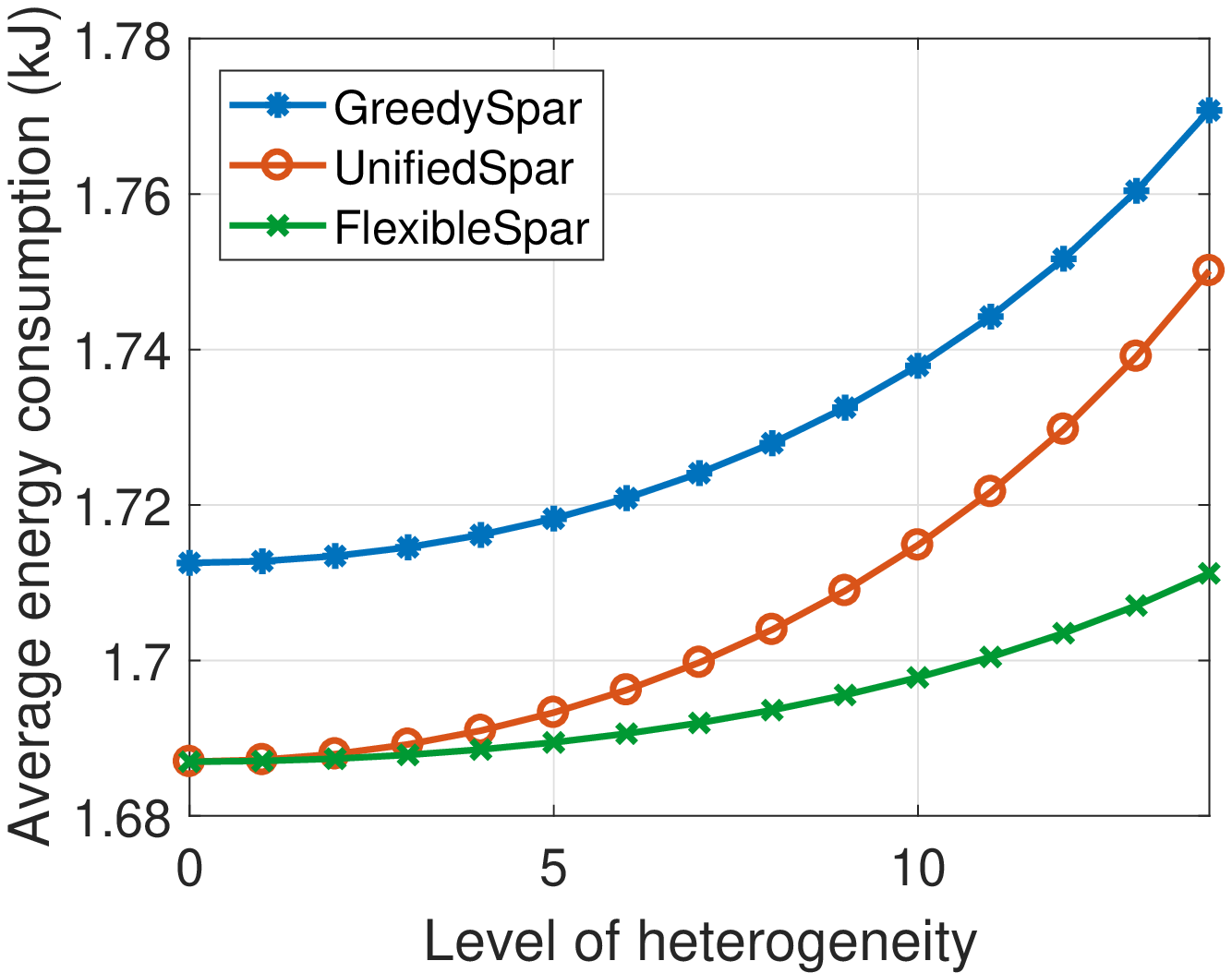}}
\subfigure[Impact of comm. capacity\label{sim3}]
  {\includegraphics[width=4.45cm]{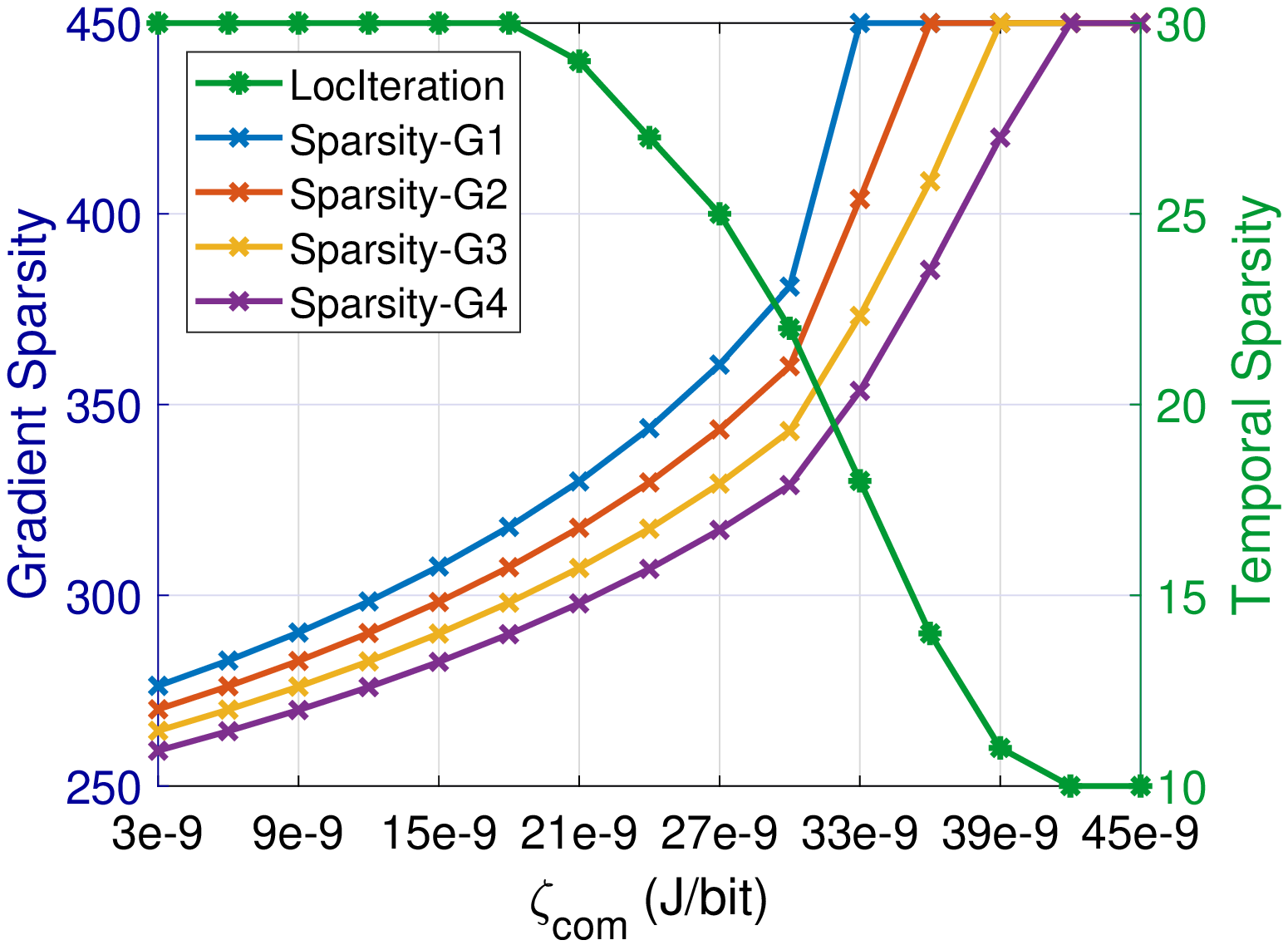}}
\subfigure[Impact of comp. capacity\label{sim4}]
  {\includegraphics[width=4.45cm]{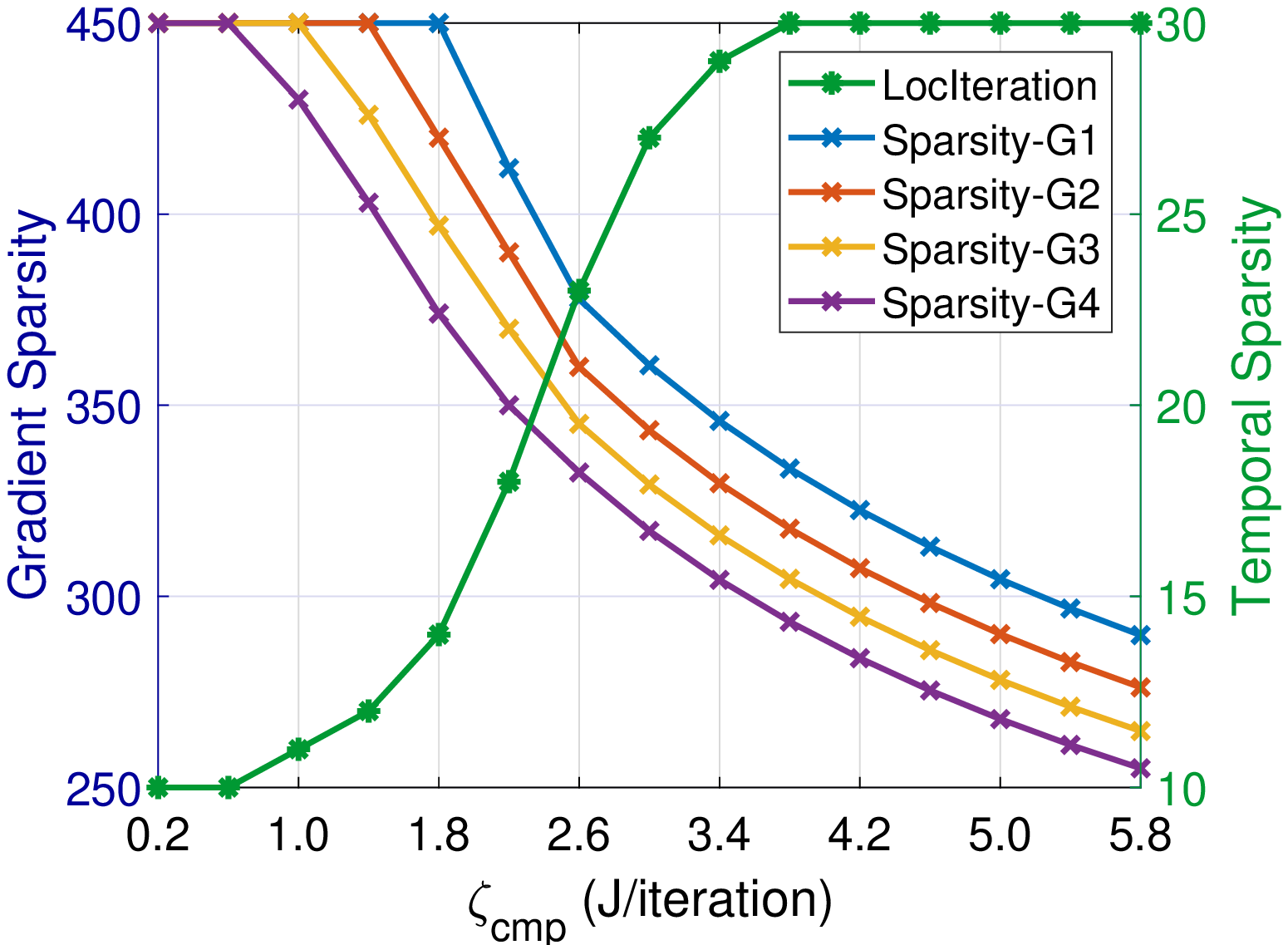}}
  \caption{Impact of system parameters.\label{Fig:SimPerformance}}
\end{figure*}

\begin{figure*} \centering 
\subfigure[Training accuracy vs epochs\label{exp1}]
  {\includegraphics[width=4.45cm]{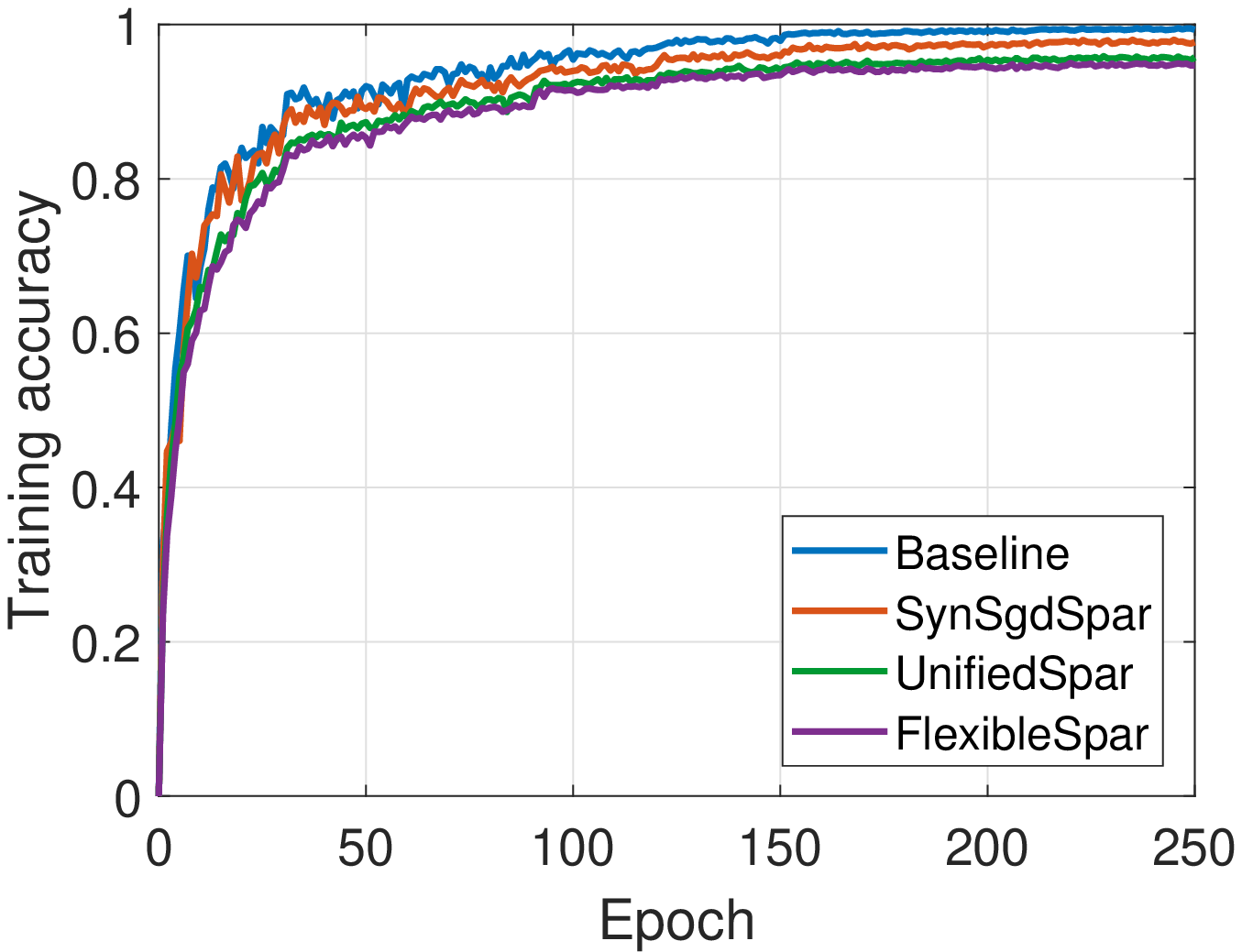}}
\subfigure[Test error vs consumed energy\label{exp2}]
  {\includegraphics[width=4.45cm]{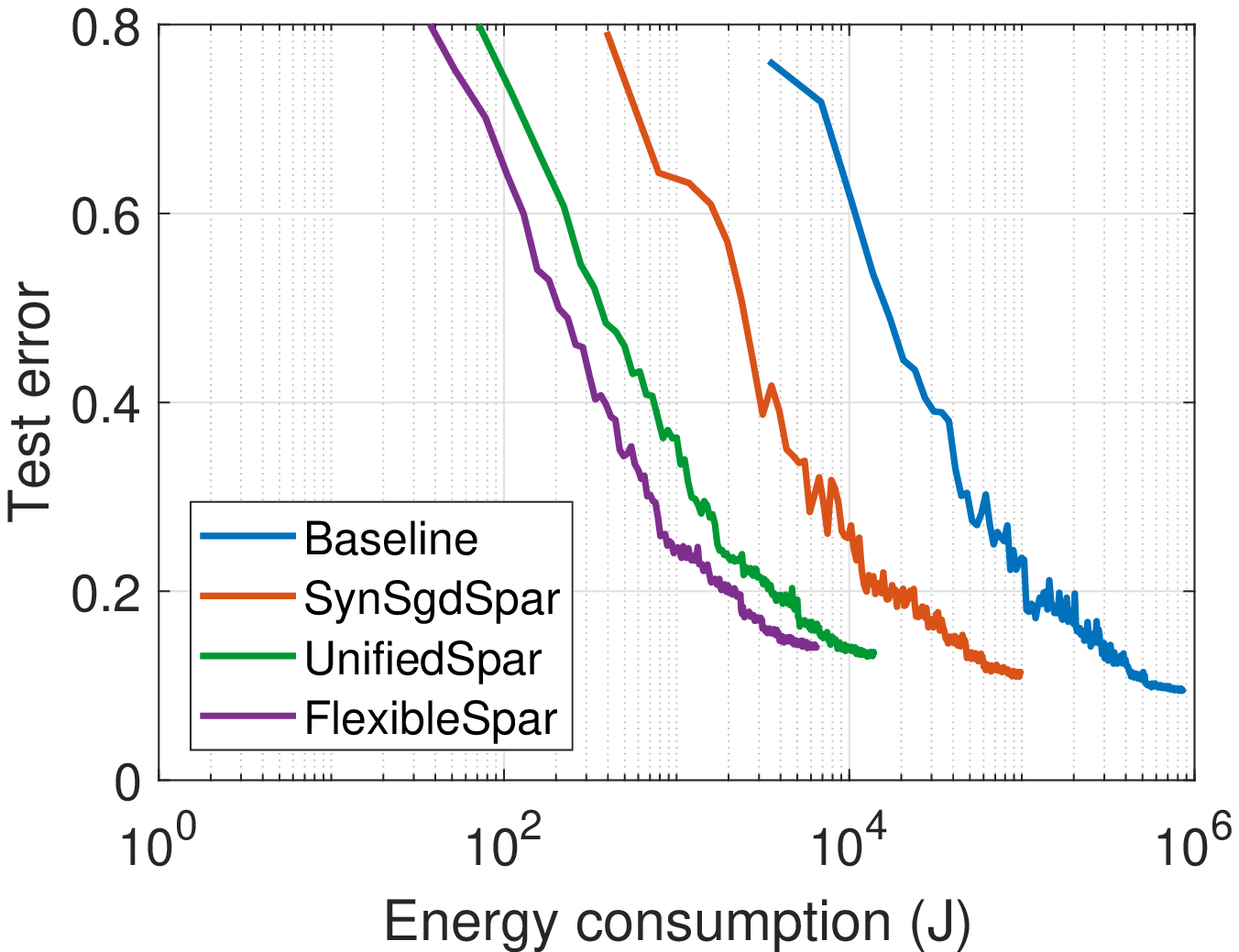}}
\subfigure[Training accuracy vs epochs\label{exp3}]
  {\includegraphics[width=4.45cm]{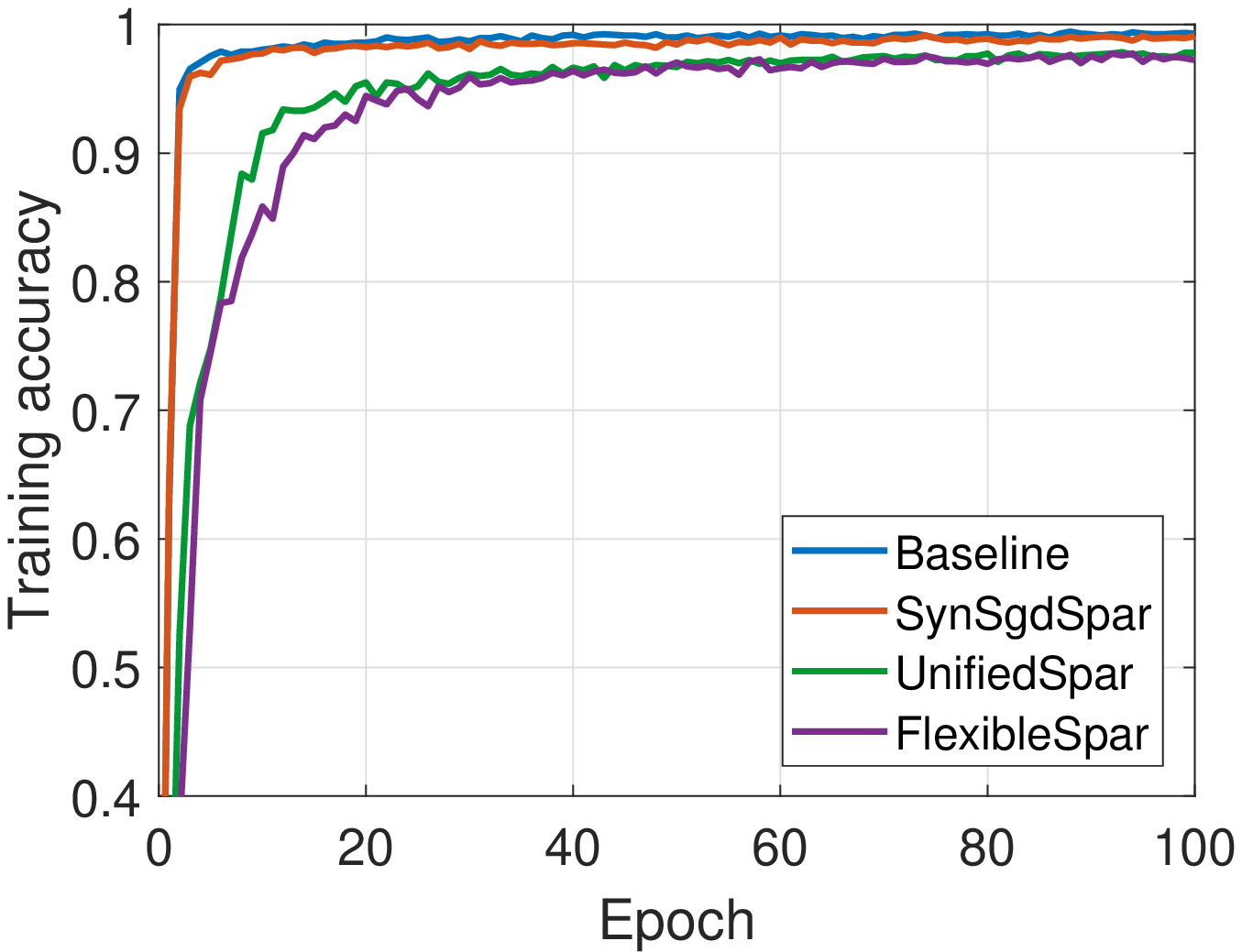}}
\subfigure[Test error vs consumed energy\label{exp4}]
  {\includegraphics[width=4.45cm]{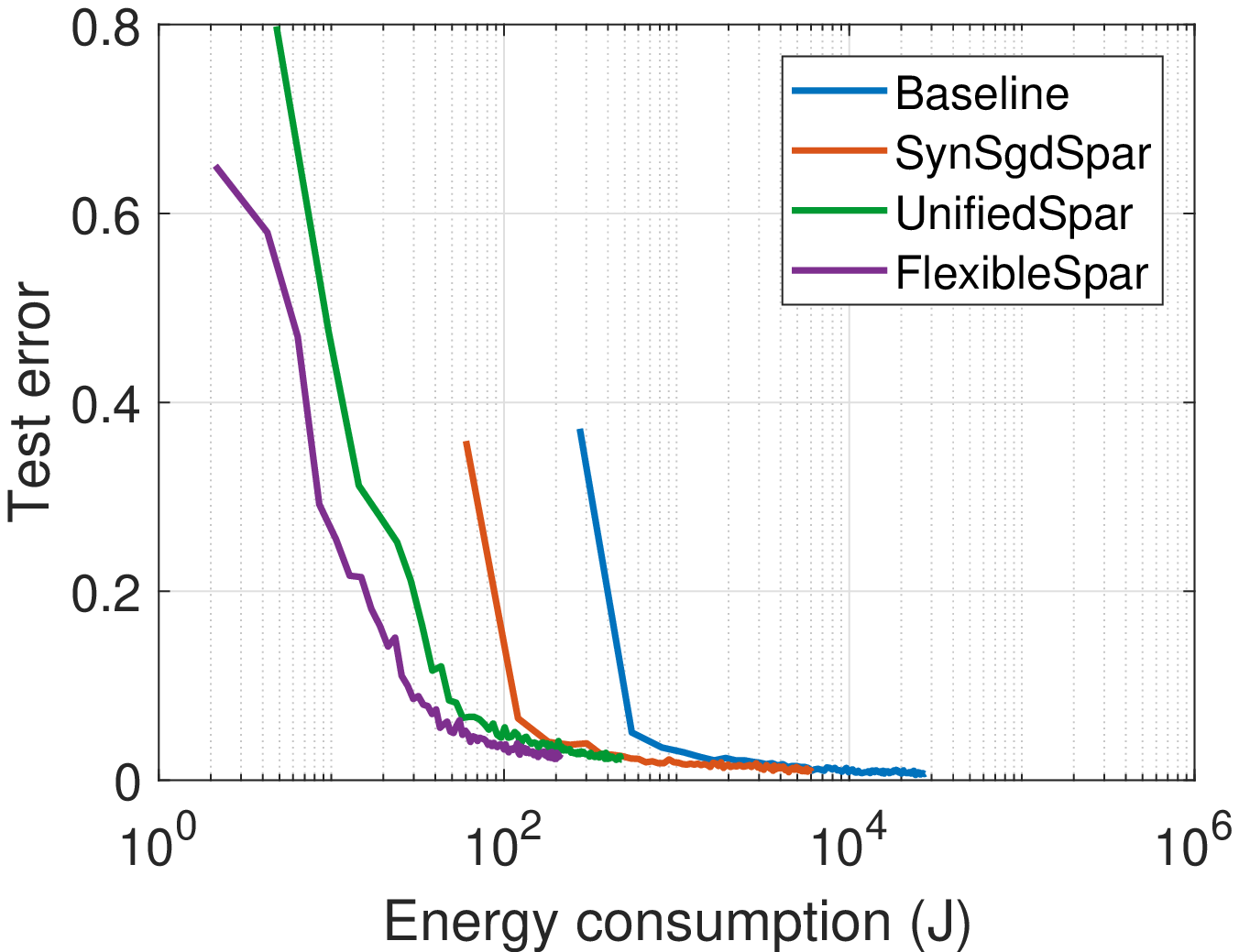}}
  \caption{Performance on various architectures and datasets. ((a-b): ResNet20 trained on CIFAR-10; (c-d): LeNet5-Caffe trained on MNIST.)\label{Fig:ExpPerformance}}
\end{figure*}

\begin{corollary}[Closed form of $\delta_m^*(\boldsymbol{\delta}^\nu),\forall m$]\label{LemmaCloseForm}
Let $\delta_m^*(\boldsymbol{\delta}^\nu)$ be the optimal solution of $\delta_m$ given the current $\boldsymbol{\delta}^\nu$. Then, each $\delta_m^*(\boldsymbol{\delta}^\nu)$ has the following expression:
\begin{equation} \label{ClosedForm}
\delta_m^*(\boldsymbol{\delta}^\nu)=\left\{
\begin{aligned}
\delta_{lb}, & & {\overline{\delta}_m \leq \delta_{lb}}\\
\overline{\delta}_m(\boldsymbol{\delta}^\nu),& &  {\delta_{lb}} < \overline{\delta}_m  \leq \delta_{ub}\\
\delta_{ub}, & & {\overline{\delta}_m \geq \delta_{ub}}
\end{aligned} \right.
\end{equation}
where $\overline{\delta}_m(\boldsymbol{\delta}^\nu)$ is given by
\begin{equation}  
\overline{\delta}_m(\boldsymbol{\delta}^\nu)\!=\!exp\left(-\frac{1}{3}W\left(-\frac{6AE\cdot exp\left(3B\right)}{DC_m}\right)+B\right)
\end{equation}
with
\begin{subequations}  
\begin{align}
& A=\alpha H^{(i)},\  B=1-\kappa\ln{2}, \  C_m=P_m s_1d/R_m, \nonumber\\
&D\!=\!\sum_m \left(\alpha H^{(i)} (\delta_m^\nu)^2+\frac{\beta}{M^{3/2}H^{(i)}}\right),\nonumber\\
& E\!=\!\sum_m \left(\frac{P_m s_1d (\log_2 \delta_m^\nu+\kappa) }{R_m \delta_m^\nu}+\frac{P_m s_0}{R_m} + E_m^{0}H^{(i)}\right)\ln 2. \nonumber 
\end{align}
\end{subequations}

Here, $W(\cdot)$ denotes the Lambert function. 
\end{corollary}
\begin{proof}
This is derived by using the Karush-Kuhn-Tucker conditions. We omit the proof due to space limitation.
\end{proof}

Thanks to the closed-formula for $\delta_m^*(\boldsymbol{\delta}^\nu)$, Algorithm \ref{algBD} can solve the inner-loop primal problem into optimality without resorting to any iterative solver that can provide approximate solutions only. We adopt the branch-and-bound algorithm to solve the master problem in (\ref{MasProb}) with computational complexity $O(2^{|\mathcal{H}|})$. When solving the primal problem takes $I_{in}$ inner-loop iterations in total, the overall complexity of Algorithm \ref{algBD} is $O(I_{out}\max\{O(2^{|\mathcal{H}|}), I_{in}\})$ in the worst case where $I_{out}$ denotes the number of iterations required by the outer loops.

\section{Performance Evaluation}
We evaluate the performance of the proposed compression control scheme, denoted by ``\textit{FlexibleSpars}'', via extensive simulations. Particularly, we compare with the following three schemes: 1) \textit{SynSgdSpars} allows participants to sparsify their gradients flexibly, and follows the typical setting of synchronous distributed gradient descent to perform global aggregation after every local update. 2) \textit{GreedySpars} greedily makes the compression control decisions by minimizing the energy cost in the current round, which is oblivious to the impact of the trade-off relationship between the number of global rounds and the cost in a single round on the total energy consumption. 3) \textit{UnifiedSpar} forces every participant to compress the gradients with a unified sparsity, regardless of the heterogeneous communication condition. Compression parameters are determined by solving a simplified version of the problem in (\ref{ProbForm2}).

Fig.~\ref{sim1} demonstrates the average energy consumption with the varying number of participating devices. We see that the average energy consumption decreases with the growing scale of the FL system under all the schemes, while all the curves tend to be flat. This is due to the fact that increasing the number of participants can help to speed up the convergence of the training process and thereby save the resources at each edge device. Yet such speed-ups will be slight when the participating devices are enough to well capture the whole dataset's information and characteristics $\bigcup_m \mathcal{D}_m$ for training. Among the four schemes, ``\textit{SynSgdSpars}'' is shown to consume the most energy since it increases the communication complexity significantly. As expected, our proposed scheme ``\textit{FlexibleSpar}'' outperforms the others as it injects more foresight than ``\textit{GreedySpar}'' and more flexibility than \textit{UnifiedSpar} into the compression decision-making, to better fit the heterogeneous communication conditions across participants. Fig.~\ref{sim2} further gives an insight into the impact of the heterogeneity level of participants' communication capacity on the system energy efficiency. Here, we set the number of participants 12 and divide them into four groups, corresponding to four capacity levels. Assume that participants belonging to the same group use the same wireless bandwidth for gradient exchanges. Let $L$ be the level of heterogeneity that controls the variations in bandwidth among different groups. Fixing the average bandwidth $\overline{W}=1$GHz, we set the bandwidth adopted by the four groups as $\overline{W}-0.03 L$ (GHz), $\overline{W}-0.01 L$ (GHz), $\overline{W}+0.01 L$ (GHz), and $\overline{W}+0.03 L$ (GHz), respectively. We set the value of $L$ to vary in $\{0,1,...,14\}$, where larger $L$ indicates higher level of heterogeneity. Fig.~\ref{sim2} elucidates that the high level of communication heterogeneity has a negative impact on FL and indeed impairs the system energy efficiency. Notice that we omit the examination of ``\textit{SynSgdSpars}'' in this setting since it performs so poorly that it is incomparable to the other schemes. Thanks to the flexible compression, our proposed scheme, as we would expect, exhibits more resilience than the others to cope with the scenario with high heterogeneity across participants in terms of wireless channel conditions.

We define $\zeta_{com}=\frac{1}{M}\sum _m \frac{P_ms_1}{R_m}$ (J/bit) and $\zeta_{cmp}=\frac{1}{M}\sum _m E_{m,ite}^{cmp}$ (J/iteration) representing the average energy intensity in terms of transmitting and computing, respectively. Keeping the above group setting of 12 participants with $L=10$, we further examine the impacts of $\zeta_{com}$ and $\zeta_{cmp}$ on the optimal values of gradient sparsity $\{\delta\}_{\forall m}$ and synchronization frequency $H$ obtained from our control algorithm. The results are shown in Fig.~\ref{sim3}-\ref{sim4}. Note that $H$ can be viewed as the level of temporal sparsity in the sense that performing multiple local iterations between every two synchronizations implicitly sparsifies the communications in the temporal domain. In Fig.~\ref{sim3}, ``\textit{Sparsity-G1/2/3/4}'' denotes the decision of gradient sparsity for the participants in the group 1/2/3/4, respectively. ``\textit{LocIteration}'' denotes the decision of their temporal sparsity $H$. When $\zeta_{com}$ is small, participants are allowed to sparsify the gradients at a relatively low degree without adding to the total energy cost significantly. In this case, the number of global rounds is dominantly affected by the second term in (\ref{GlobRound}), and a large temporal sparsity $H$ is needed to reduce the number of necessary global rounds. As $\zeta_{com}$ grows, all the participants tend to increase the gradient sparsity to fit the worsening communication conditions. We observe that the participants in Group 1 suffering the worst channel conditions generally prefer higher gradient sparsity than the participants in the other groups. When $\zeta_{com}$ is large enough, the impact of communication on the total energy consumption becomes more profound than that of computing, forcing the participants to compress the gradients severely with the large $\{\delta_m\}_{\forall m}$ to alleviate the communication burden. Accordingly, the first term in (\ref{GlobRound}) begins to take effects, resulting in the decreasing degree of temporal sparsity $H$. The analysis above can also be verified by Fig.~\ref{sim4}, where we vary $\zeta_{cmp}$ while keeping $\zeta_{com}$ unchanged. We find that the curves in Fig.~\ref{sim4} are somewhat symmetrical to the curves in Fig.~\ref{sim3}. The reason lies in that increasing $\zeta_{cmp}$ can be viewed as decreasing $\zeta_{com}$ in our system, both of which imply the process of computing cost becoming the bottleneck. Fig.~\ref{sim3}-\ref{sim4} reveal that there is a trade-off between gradient sparsity and temporal sparsity, which indeed corresponds to less ``talking'' and less ``working''. As expected, our flexible compression scheme allows participants to balance these two types of sparsity smoothly against one another for saving energy during training.

Fig.~\ref{Fig:ExpPerformance} further presents the training results on several commonly used deep models and datasets. Specifically, Fig.~\ref{exp1}-\ref{exp2} show the convergence rate in terms of epochs and consumed energy respectively for ResNet20 \cite{he2016deep} trained on CIFAR-10, while Fig.~\ref{exp3}-\ref{exp4} show the same things for LeNet5-Caffe \cite{sattler2019sparse} trained on MNIST. Here, we take ordinary distributed SGD as the baseline in which each participant transmits full gradients to the server after every local update. From Fig.~\ref{exp1} and Fig.~\ref{exp3}, we observe that ``\textit{FlexibleSpar}'' exhibits very similar behavior with ``\textit{UnifiedSpar}'' in terms of convergence rate and final accuracy, both of which slightly underperform ``\textit{SynSgdSpar}''. This also implies that temporal sparsity has a more profound impact on convergence rate than gradient sparsity in our setting. Due to delayed synchronization and imprecise gradient information, both ``\textit{UnifiedSpar}'' and ``\textit{FlexibleSpar}'' are shown to slow down the convergence at initial, which is consistent with our convergence analysis. In spite of this, ``\textit{FlexibleSpar}'' is validated to be capable of saving energy for on-device training. As reported in Fig.~\ref{exp2} and Fig.~\ref{exp4}, ``\textit{FlexibleSpar}'' consumes $\times 1.5\!-\!\times 100$ less energy than the other schemes to reach a given target accuracy.

\section{Conclusion}
In this work, we have presented a holistic communication compression solution to reduce the energy consumption of FL over heterogeneous participating edge devices without sacrificing the model accuracy. We have developed a FL algorithm enabling flexible communication compression and provided the convergence analysis from a theoretical perspective. Considering the heterogeneous computing and communication conditions across edge devices, we have further designed an energy-efficiency oriented compression control scheme guided by the derived convergence bound. Extensive simulations have been conducted to verify the theoretical analysis and evaluate the algorithm's performance. The results have shown that our flexibly compressed FL scheme exhibits great potentials in accommodating heterogeneous mobile edge devices and improving the energy efficiency of FL over those edge devices.



\section*{Acknowledgment}
The work of L. Li, R. Hou and H. Li was partially supported by National Natural Science Foundation of China (Grant No. 61571351), State Key Laboratory of Computer Architecture (ICT, CAS) under Grant No. CARCH201904, the Major Research plan of the Shaanxi Science Foundation of China (2019ZDLGY12-08), the 111 project (grant No. B16037), and OPPO funding. The work of D. Shi and M. Pan was supported in part by the U.S. National Science Foundation under grants US CNS-1646607, CNS-1801925, and CNS-2029569. The work of Z. Han was partially supported by NSF EARS-1839818, CNS-1717454, CNS-1731424, and CNS-1702850.



\section*{Appendix}
\subsection{Proof of Theorem 1}
Inspired by the perturbed iterate analysis framework in \cite{basu2019qsparse}, we first define the following auxiliary sequences for all $t\geq 0$:

1) 
\begin{equation}\label{VirtSeq}
\tilde{\boldsymbol{w}}_m^{(t)}=\left\{
\begin{aligned}
\widehat{\boldsymbol{w}}_m^{(0)}\quad \quad \quad \quad \quad \quad\quad \quad \quad \quad \quad \quad \ , & & {t=0}\\
\tilde{\boldsymbol{w}}_m^{(t-1)}-\eta \nabla f_m(\widehat{\boldsymbol {w}}_m^{(t-1)};\mathcal{D}_m^{(t-1)})\ \ , & & {t \geq 1}
\end{aligned} \right.
\end{equation}

2) $\boldsymbol{q}^{(t)}\triangleq \frac{1}{M} \sum \limits_{m=1}^M  \nabla f_m(\widehat{\boldsymbol {w}}_m^{(t)};\mathcal{D}_m^{(t)})$

3)  $ \overline{\boldsymbol{q}}^{(t)}\triangleq \mathbb{E}_{\mathcal{D}_m^{(t)}}[\boldsymbol{q}^{(t)}]= \frac{1}{M} \sum \limits_{m=1}^M  \nabla f_m(\widehat{\boldsymbol {w}}_m^{(t)})$

4) $\tilde{\boldsymbol{w}}^{(t)}\triangleq \frac{1}{M} \sum \limits_{m=1}^M  \nabla f_m(\widehat{\boldsymbol {w}}_m^{(t)};\mathcal{D}_m^{(t)})=\tilde{\boldsymbol{w}}^{(t-1)}-\eta \boldsymbol{q}^{(t-1)}$

5) $\widehat{\boldsymbol {w}}^{(t)}= \frac{1}{M} \sum \limits_{m=1}^M  \widehat{\boldsymbol {w}}_m^{(t)}$

By the smoothness of $F: \mathbb{R}^d \rightarrow \mathbb{R}$, we have
\begin{equation} \label{step4}
\begin{split}
& F(\tilde{\boldsymbol{w}}^{(t+1)})- F(\tilde{\boldsymbol{w}}^{(t)}) \\
\leq & <\nabla F(\tilde{\boldsymbol{w}}^{(t)}), \tilde{\boldsymbol{w}}^{(t+1)}-\tilde{\boldsymbol{w}}^{(t)}> + \frac{L}{2}||\tilde{\boldsymbol{w}}^{(t+1)}-\tilde{\boldsymbol{w}}^{(t)}||^2 \\
= &-\eta<\nabla F(\tilde{\boldsymbol{w}}^{(t)}),\boldsymbol{q}^{(t)}>+\frac{\eta^2L}{2}||\boldsymbol{q}^{(t)}||^2\\
\overset{(a)}\leq &-\eta<\nabla F(\tilde{\boldsymbol{w}}^{(t)}),\boldsymbol{q}^{(t)}>+\eta^2L||\boldsymbol{q}^{(t)}-\overline{\boldsymbol{q}}^{(t)}||^2+\eta^2L||\overline{\boldsymbol{q}}^{(t)}||^2 \\
=&-\frac{\eta}{M}\sum\limits_{m=1}^M <\nabla F(\tilde{\boldsymbol{w}}^{(t)}),\nabla f_m(\widehat{\boldsymbol {w}}_m^{(t)};\mathcal{D}_m^{(t)})> \\
&+ \eta^2L||\boldsymbol{q}^{(t)}-\overline{\boldsymbol{q}}^{(t)}||^2 + \eta^2L||\frac{1}{M}\sum\limits_{m=1}^M\nabla f_m(\widehat{\boldsymbol {w}}_m^{(t)};\mathcal{D}_m^{(t)})||^2,
\end{split}
\end{equation}
where $(a)$ is by Jensen’s inequality. Taking expectation with respect to the sampling mini-batch $\mathcal{D}_m^{(t)}$ by each edge device at time $t$ gives
\begin{equation} \label{step5}
\begin{split}
& \mathbb{E}[F(\tilde{\boldsymbol{w}}^{(t+1)})]- F(\tilde{\boldsymbol{w}}^{(t)}) \\
\overset{(a)} \leq & -\frac{\eta}{2}(||\nabla F(\tilde{\boldsymbol{w}}^{(t)})||^2+||\frac{1}{M}\sum\limits_{m=1}^M\nabla f_m(\widehat{\boldsymbol{w}}_m^{(t)})||^2) \\
&+\frac{\eta}{2}||\nabla F(\tilde{\boldsymbol{w}}^{(t)})-\frac{1}{M}\sum\limits_{m=1}^M\nabla f_m(\widehat{\boldsymbol{w}}_m^{(t)})||^2 \\
&+\eta^2L||\frac{1}{M}\sum\limits_{m=1}^M\nabla f_m(\widehat{\boldsymbol{w}}_m^{(t)})||^2 +\frac{\eta^2L\sigma^2}{Mb^{(t)}} \\
\overset{(b)} \leq & -\frac{\eta}{2M}\sum\limits_{m=1}^M(||\nabla F(\tilde{\boldsymbol{w}}^{(t)})||^2-L^2||\tilde{\boldsymbol{w}}^{(t)}-\widehat{\boldsymbol{w}}_m^{(t)}||^2) \\
&+\frac{2\eta^2L-\eta}{2}||\frac{1}{M}\sum\limits_{m=1}^M\nabla f_m(\widehat{\boldsymbol{w}}_m^{(t)})||^2+\frac{\eta^2L\sigma^2}{Mb^{(t)}}\\
=& -\frac{\eta}{2M}\sum\limits_{m=1}^M(||\nabla F(\tilde{\boldsymbol{w}}^{(t)})||^2+L^2||\tilde{\boldsymbol{w}}^{(t)}-\widehat{\boldsymbol{w}}_m^{(t)}||^2) \\
&+\frac{2\eta^2L-\eta}{2M}\sum\limits_{m=1}^M||\nabla f_m(\widehat{\boldsymbol{w}}_m^{(t)})||^2+\frac{\eta^2L\sigma^2}{Mb^{(t)}}\\
&+\frac{\eta L^2}{M}||\tilde{\boldsymbol{w}}^{(t)}-\widehat{\boldsymbol{w}}_m^{(t)}||^2
\end{split}
\end{equation}
where $(a)$ follows by applying two basic inequalities $<\boldsymbol{a},\boldsymbol{a}>\leq 1/2 ||\boldsymbol{a}||^2+1/2 ||\boldsymbol{b}||^2$ and $\mathbb{E}[||X||^2]=\mathbb{E}[||X-\mathbb{E}[X]||^2]+||\mathbb{E}[X]||^2$; $(a)$ follows from the lipschitz continuity of the gradient of local functions. The first term in (\ref{step5}) can be bounded in terms of $||\nabla f_m(\widehat{\boldsymbol{w}}_m^{(t)})||^2$ as follows:
\begin{equation} \label{step6}
\begin{split}
& ||\nabla f_m(\widehat{\boldsymbol{w}}_m^{(t)})||^2 \\
\leq & 2||\nabla f_m(\widehat{\boldsymbol{w}}_m^{(t)})-\nabla F(\tilde{\boldsymbol{w}}^{(t)})||^2+2||F(\tilde{\boldsymbol{w}}^{(t)})||^2 \\
\leq & 2L^2 ||\widehat{\boldsymbol{w}}_m^{(t)}-\tilde{\boldsymbol{w}}^{(t)}||^2+||F(\tilde{\boldsymbol{w}}^{(t)})||^2\\
\end{split}
\end{equation}

Using $\eta\leq \frac{1}{2L}$ and rearranging the terms in (\ref{step5}), we have
\begin{equation} \label{step7}
\begin{split}
&\frac{\eta}{4M}\sum\limits_{m=1}^M||\nabla f_m(\widehat{\boldsymbol{w}}_m^{(t)})||^2\\
\leq & F(\tilde{\boldsymbol{w}}^{(t)})-\mathbb{E}[F(\tilde{\boldsymbol{w}}^{(t+1)})]+\frac{\eta^2L\sigma^2}{Mb^{(t)}}+\frac{\eta L^2}{M}||\tilde{\boldsymbol{w}}^{(t)}-\widehat{\boldsymbol{w}}_m^{(t)}||^2
\end{split}
\end{equation}

Taking expectation with respect to the entire process and using the basic inequality $||\boldsymbol{a}+\boldsymbol{b}||^2\leq 2||\boldsymbol{a}||^2+2||\boldsymbol{b}||^2$ gives
\begin{equation} \label{step8}
\begin{split}
&\frac{\eta}{4M}\sum\limits_{m=1}^M\mathbb{E}[||\nabla f_m(\widehat{\boldsymbol{w}}_m^{(t)})||^2]\\
\leq & \mathbb{E}[F(\tilde{\boldsymbol{w}}^{(t)})]-\mathbb{E}[F(\tilde{\boldsymbol{w}}^{(t+1)})]+\frac{\eta^2L\sigma^2}{Mb^{(t)}}\\
&+2\eta L^2\mathbb{E}[||\tilde{\boldsymbol{w}}^{(t)}-\widehat{\boldsymbol{w}}^{(t)}||^2]
+\frac{2\eta L^2}{M}\sum\limits_{m=1}^M\mathbb{E}[||\widehat{\boldsymbol{w}}^{(t)}-\widehat{\boldsymbol{w}}_m^{(t)}||^2]\\
\overset{(a)}\leq & \mathbb{E}[F(\tilde{\boldsymbol{w}}^{(t)})]-\mathbb{E}[F(\tilde{\boldsymbol{w}}^{(t+1)})]+\frac{2\eta^2L\sigma^2}{M\rho^tb^{(0)}}\\
&+2\eta L^2\mathbb{E}[||\tilde{\boldsymbol{w}}^{(t)}-\widehat{\boldsymbol{w}}^{(t)}||^2]
+\frac{2\eta L^2}{M}\sum\limits_{m=1}^M\mathbb{E}[||\widehat{\boldsymbol{w}}^{(t)}-\widehat{\boldsymbol{w}}_m^{(t)}||^2],
\end{split}
\end{equation}
where $(a)$ follows by recalling $b^{(t)}=\lfloor \rho^t b^{(0)} \rfloor$ and noting $\lfloor x \rfloor>x/2$ as long as $x\geq 2$.

Now we give three important lemmas where the first two are borrowed from \cite{basu2019qsparse} and the last one is proved in the following.   

\begin{lemma}[Memory~\cite{basu2019qsparse}]\label{lemma1}
The accumulated error captures the distance between the true sequence and virtual sequence. That is
\begin{equation}
    \widehat{\boldsymbol{w}}^{(t)}-\tilde{\boldsymbol{w}}^{(t)}=\frac{1}{M}\sum \limits_{m-1}^M \boldsymbol{e}_m^{(t)}
\end{equation}
\end{lemma}

\begin{lemma}[Contracting Deviation of Local Sequences~\cite{basu2019qsparse}]\label{lemma2}
The deviation of the local sequences is bounded by 
\begin{equation}
    \frac{1}{M} \sum \limits_{m-1}^M \mathbb{E}[||\widehat{\boldsymbol{w}}^{(t)}-\widehat{\boldsymbol{w}}_m^{(t)}||^2] \leq \eta ^2 G^2 H^2
\end{equation}
\end{lemma}

\begin{lemma}[Bounded Memory]\label{lemma3}
For worker $m$ who synchronizes with the server every $H$ local iterations, we have
\begin{equation}
    \mathbb{E} [||e_m^{(t)}||^2] \leq 4 \delta_m^2 \eta^2G^2H^2
\end{equation}
\end{lemma}

\begin{proof}
Note that Algorithm \ref{alg:topkSGD} average the gradients every $H$ iterations between which the accumulated error $\boldsymbol{e}_m^{(t)}$ at any participant $m$ and the global parameter vector $\boldsymbol{w}^{(t)}$ keep unchanged. For ease of presentation, we assume that $T$ is an integer multiple of $H$. Let $\mathcal{I}_T=\{t_1,t_2,...,t_{T/H}=T\}$ be the aggregation indices satisfying $t_{i+1}-t_{i}=H$. For every $m\in\mathcal{M}$, we have
\begin{equation} \label{step1}
\begin{split}
&\mathbb{E}[||\boldsymbol{e}_m^{(t_{i+1})}||^2] \\
=&\mathbb{E}[||\boldsymbol{e}_m^{(t_{i+1}-1)}+\boldsymbol{w}^{(t_{i+1}-1)}-\widehat{\boldsymbol{w}}_m^{(t_{i+1}-\frac{1}{2})}-\boldsymbol{g}_m^{(t_{i+1}-1)}||^2] \\
\overset{(a)}\leq & (1-\frac{1}{\delta_m})\mathbb{E}[||\boldsymbol{e}_m^{(t_{i+1}-1)}+\boldsymbol{w}^{(t_{i+1}-1)}-\widehat{\boldsymbol{w}}_m^{(t_{i+1}-\frac{1}{2})}||^2]\\
\overset{(b)}= & (1-\frac{1}{\delta_m})\mathbb{E}[||\boldsymbol{e}_m^{(t_{i})}+\widehat{\boldsymbol{w}}_m^{(t_{i})}-\widehat{\boldsymbol{w}}_m^{(t_{i+1}-\frac{1}{2})}||^2].       
\end{split}
\end{equation}
Here $(a)$ is due to the contraction property of ${\rm Top}_k(\boldsymbol{x})$ operator~\cite{stich2018sparsified}, that is $\mathbb{E}||\boldsymbol{x}-{\rm Top}_k(\boldsymbol{x})||^2\leq(1-k/d)||\boldsymbol{x}||^2,\ \forall \boldsymbol{x}\in \mathbb{R}^d$. In $(b)$, we use $\boldsymbol{e}^{(t_{i+1}-1)}=\boldsymbol{e}^{(t_{i})}$ and $\boldsymbol{w}^{(t_{i+1}-1)}=\boldsymbol{w}^{(t_{i})}=\widehat{\boldsymbol{w}}_m^{(t_{i})}$ that always hold. Since the inequality $||\boldsymbol{a}+\boldsymbol{b}||^2\leq(1+\tau)||\boldsymbol{a}+(1+\frac{1}{\tau})||\boldsymbol{b}||^2$ holds for every $\tau\geq 0$, we take any $p>1$ and transform (\ref{step1}) as follows
\begin{equation} \label{step2}
\begin{split}
&\mathbb{E}[||\boldsymbol{e}_m^{(t_{i+1})}||^2] \\
\leq & (1-\frac{1}{\delta_m})\{(1+\frac{(p-1)}{p \delta_m})\mathbb{E}[||\boldsymbol{e}_m^{(t_{i})}||^2]\\
&+(1+\frac{p \delta_m}{(p-1)})\mathbb{E}[||\widehat{\boldsymbol{w}}_m^{(t_{i})}-\widehat{\boldsymbol{w}}_m^{(t_{i+1}-\frac{1}{2})}||^2]\}\\
\leq & (1-\frac{1}{p\delta_m})\mathbb{E}[||\boldsymbol{e}_m^{(t_{i})}||^2] \\
&+ \frac{p(\delta_m^2-1)}{(p-1)\delta_m}\mathbb{E}[||\sum \limits_{j=t_i}^{t_{i+1}-1}\eta\nabla f_m(\boldsymbol {w}_m^{(j)};\mathcal{D}_m^{(j)})||^2]\\
\overset{(a)}\leq & (1-\frac{1}{p\delta_m})\mathbb{E}[||\boldsymbol{e}_m^{(t_{i})}||^2]+\frac{p(\delta_m^2-1)}{(p-1)\delta_m}\eta^2G^2H^2,       
\end{split}
\end{equation}
where $(a)$ follows from Assumption 1. Iterating the above inequality from $i=0 \rightarrow l$ where $l=T/H$ yields:
\begin{equation} \label{step3}
\begin{split}
&\mathbb{E}[||\boldsymbol{e}_m^{(t_{i+1})}||^2] \\
\leq & \frac{p(\delta_m^2-1)}{(p-1)\delta_m}\eta^2G^2H^2 \sum \limits_{j=1}^{(l)}(1-\frac{1}{p\delta_m})^{l-j} \\
\overset{(a)} \leq & \frac{p^2(\delta_m^2-1)}{p-1}\eta^2G^2H^2 \\
\overset{(b)} \leq & 4\delta_m^2\eta^2G^2H^2 ,
\end{split}
\end{equation}
where $(a)$ is by the fact that $\sum _{j=1}^{(l)}(1-\frac{1}{p\delta_m})^{l-j}\leq \sum_{j\geq 0}(1-\frac{1}{p\delta_m})^{j}=p\delta_m$, and $(b)$ is by plugging $p=2$. Note the the right-hand-side does not depend on $t$, i.e., for every $t=0,1,...,T$, the following holds:
\begin{equation}
    \mathbb{E}[||\boldsymbol{e}_m^{(t)}||^2]\leq 4\delta_m^2\eta^2G^2H^2.
\end{equation}
\end{proof}

Lemma \ref{lemma1} and Lemma \ref{lemma3} together imply:
\begin{equation}\label{eq:l1l3}
    \mathbb{E}[||\widehat{\boldsymbol{w}}^{(t)}-\tilde{\boldsymbol{w}}^{(t)}||^2]\leq \frac{4\eta^2G^2H^2}{M}\sum\limits_{m=1}^M\delta_m^2.
\end{equation}

Applying Lemma \ref{lemma2} and (\ref{eq:l1l3}) into (\ref{step8}), we get 
\begin{equation} \label{SingleStep}
\begin{split}
&\frac{\eta}{4M}\sum\limits_{m=1}^M\mathbb{E}[||\nabla f_m(\widehat{\boldsymbol{w}}_m^{(t)})||^2]\\
\leq &  \mathbb{E}[F(\tilde{\boldsymbol{w}}^{(t)})]-\mathbb{E}[F(\tilde{\boldsymbol{w}}^{(t+1)})]+\frac{2\eta^2L\sigma^2}{M\rho^tb^{(0)}}\\
&+\frac{8\eta^3 L^2G^2H^2}{M}\sum\limits_{m=1}^M \delta_m^2
+2\eta^3 L^2G^2H^2,
\end{split}
\end{equation}

Recursively applying the above inequality from $t=0$ to $t=T-1$ yields
\begin{equation} \label{sum}
\begin{split}
&\frac{1}{4MT}\sum\limits_{t=0}^{T-1}\sum\limits_{m=1}^M\mathbb{E}[||\nabla f_m(\widehat{\boldsymbol{w}}_m^{(t)})||^2]\\
\leq & \frac{\mathbb{E}[F(\tilde{\boldsymbol{w}}^{(0)})]-F^*}{\eta T}+\frac{2\eta L\sigma^2}{M T b^{(0)}}\sum\limits_{t=0}^{T-1}\frac{1}{\rho^{t}}\\
&+\frac{8\eta^2 L^2G^2H^2}{M}\sum\limits_{m=1}^M \delta_m^2
+2\eta^2 L^2G^2H^2 \\
\overset{(a)} \leq & \frac{\mathbb{E}[F(\tilde{\boldsymbol{w}}^{(0)})]-F^*}{\eta T}+\frac{2\eta \rho L\sigma^2}{(\rho-1)M T b^{(0)}}\\
&+\frac{8\eta^2 L^2G^2H^2}{M}\sum\limits_{m=1}^M \delta_m^2
+2\eta^2 L^2G^2H^2,
\end{split}
\end{equation}
where $(a)$ follows by simplifying the partial sum of geometric series and noting that $0<\frac{1}{\rho}<1$. Let $\boldsymbol{z}_T$ be a random variable sampled from $\{\widehat{\boldsymbol{w}}_m^{(t)}\}$ with probability ${\rm Pr}[\boldsymbol{z}_T=\widehat{\boldsymbol{w}}_m^{(t)}]=\frac{1}{MT}$. By taking $\delta=\sqrt{\frac{1}{M}\sum\limits_{m=1}^M \delta_m^2}$ and $\eta=\frac{\theta\sqrt{M}}{\sqrt{T}}$( where $\theta$ is a constant satisfying $\frac{\theta\sqrt{M}}{\sqrt{T}}\leq\frac{1}{2L}$), we have
\begin{equation} \label{final}
\begin{split}
\mathbb{E}[||\boldsymbol{z}_T||^2] =&\frac{1}{MT}\sum\limits_{t=0}^{T-1}\sum\limits_{m=1}^M\mathbb{E}[||\nabla f_m(\widehat{\boldsymbol{w}}_m^{(t)})||^2] \\
\leq& \frac{4(\mathbb{E}[F(\boldsymbol{w}^{(0)})]-F^*)}{\theta \sqrt{MT}}+\frac{8\rho\theta L \sigma^2}{(\rho-1)Mb^{(0)} \sqrt{M}T^{3/2}}\\
&+(4\delta^2+1)\frac{8M\theta^2L^2G^2H^2}{T},
\end{split}
\end{equation}

Until now we complete the proof of Theorem 1

\end{document}